\documentclass{article}

\PassOptionsToPackage{numbers, compress}{natbib}

\usepackage[acronym,smallcaps,nowarn]{glossaries}
\glsdisablehyper{}
\newacronym{RL}{rl}{reinforcement learning}
\newacronym{SIL}{sil}{self-imitation learning}
\newacronym{PPO}{ppo}{proximal policy optimization}
\newacronym{TD3}{td3}{twin-delayed deep deterministic policy gradient}
\newacronym{DDPG}{ddpg}{deep deterministic policy gradient}
\newacronym{TD}{td}{temporal difference}

\newacronym{EM}{em}{expectation maximization}
\newacronym{IS}{is}{importance sampling}

\newacronym{HER}{her}{hindsight experience replay}
\newacronym{HPG}{hpg}{hindsight policy gradient}
\newacronym{hEM}{$\mathrm{h}$em}{hindsight expectation maximization}

\newacronym{MDP}{mdp}{markov decision process}
\newacronym{Variational RL}{variational rl}{variational reinforcement learning}

\newacronym{ELBO}{elbo}{evidence lower bound}

 \usepackage[final]{neurips_2021}
\usepackage[utf8]{inputenc} % allow utf-8 input
\usepackage[T1]{fontenc}    % use 8-bit T1 fonts
\usepackage{hyperref}       % hyperlinks
\usepackage{url}            % simple URL typesetting
\usepackage{booktabs}       % professional-quality tables
\usepackage{amsfonts}       % blackboard math symbols
\usepackage{nicefrac}       % compact symbols for 1/2, etc.
\usepackage{microtype}      % microtypography
\usepackage{amsmath}
\usepackage{algorithm,algorithmic}
\usepackage{graphicx}
\usepackage{amsmath,amssymb,amsthm}
\usepackage{algorithm,algorithmic}
\usepackage{subfigure}

\usepackage{xcolor}
\usepackage{framed}
\usepackage{graphicx}
\usepackage{mathtools}
\graphicspath{ {images/} }
\usepackage{tikz}
\usetikzlibrary{arrows,shapes,snakes,automata,backgrounds,petri}
\usepackage{pdfpages}
\usepackage{hyperref}
\hypersetup{
	colorlinks=true,
	linkcolor=blue,
	filecolor=magenta,      
	urlcolor=cyan,
}

\usepackage[T1]{fontenc}
\usepackage[utf8]{inputenc}
\usepackage{tabularx,ragged2e,booktabs,caption}
\newcolumntype{C}[1]{>{\Centering}m{#1}}

\usepackage{caption}

\usepackage{enumitem}

\usepackage{thmtools}

\usepackage{hyperref}
\usepackage{cleveref}

\usepackage{mathrsfs}
\usepackage{mathtools}

\usepackage{enumitem}
\usepackage{natbib}
\usepackage{verbatim}
\usepackage{url}
\usepackage{thm-restate}

\usepackage{wrapfig}

\usepackage{bbm}

\usepackage{parskip}

\usepackage{graphbox}

\usepackage{subfigure}

\newtheoremstyle{definition}
{3pt} % Space above
{3pt} % Space below
{} % Body font
{} % Indent amount
{\bfseries} % Theorem head font
{.} % Punctuation after theorem head
{.5em} % Space after theorem head
{} % Theorem head spec (can be left empty, meaning `normal')

\theoremstyle{definition}

\renewcommand{\epsilon}{\varepsilon}
\renewcommand{\hat}{\widehat}
\renewcommand{\tilde}{\widetilde}
\renewcommand{\bar}{\overline}

\title{Unifying Gradient Estimators for Meta-Reinforcement Learning via Off-Policy Evaluation}

\author{%
   Yunhao Tang\** \\
  Columbia University\\
  \texttt{yt2541@columbia.edu}
  \And
  Tadashi Kozuno\** \\
  University of Alberta\\
  \texttt{tadashi.kozuno@gmail.com }
  \And
  Mark Rowland\\
  DeepMind London\\
  \texttt{markrowland@deepmind.com }
  \And
  R\'emi Munos\\
  DeepMind Paris\\
  \texttt{munos@deepmind.com }
  \And
  Michal Valko\\
  DeepMind Paris\\
  \texttt{valkom@deepmind.com }
  }

\begin{document}

\maketitle

\begin{abstract}
Model-agnostic meta-reinforcement learning requires estimating the Hessian matrix of value functions. This is challenging from an implementation perspective, as repeatedly differentiating policy gradient estimates may lead to biased Hessian estimates. In this work, we provide a unifying framework for estimating higher-order derivatives of value functions, based on off-policy evaluation. Our framework interprets a number of prior approaches as special cases and elucidates the bias and variance trade-off of Hessian estimates. This framework also opens the door to a new family of estimates, which can be easily implemented with auto-differentiation libraries, and lead to performance gains in practice.
We open source the code to reproduce our results\footnote{\url{https://github.com/robintyh1/neurips2021-meta-gradient-offpolicy-evaluation}}. \end{abstract}

\section{Introduction}

Recent years have witnessed the success of reinforcement learning (RL) in challenging domains such as game playing \citep{mnih2013}, board games \citep{silver2016}, and robotics control \citep{levine2016}. However, despite such breakthroughs, state-of-the-art RL algorithms are still plagued by sample inefficiency, as training agents requires orders of magnitude more samples than humans would experience to reach similar levels of performance \citep{mnih2013,silver2016}. One hypothesis on the source of such inefficiencies is that standard RL algorithms are not good at leveraging prior knowledge, which implies that whenever presented with a new task, the algorithms must learn from scratch. On the contrary, humans are much better at transferring prior skills to new scenarios, an innate ability arguably obtained through evolution over thousand of years.

Meta-reinforcement learning (meta-RL) formalizes the learning and transfer of prior knowledge in RL~\citep{finn2017model}. The high-level idea is to have an RL agent that interacts with a distribution of environments at meta-training time. The objective is that at meta-testing time, when the agent interacts with previously unseen environments, it can learn much faster than meta-training time.
Here, faster learning is measured by the number of new samples needed to achieve a good level of performance.
If an agent can achieve good performance at meta-testing time, it  embodies the ability to transfer knowledge from prior experiences at meta-training time. meta-RL algorithms can be constructed in many ways, such as those based on recurrent memory \citep{wang2016,duan2016rl}, gradient-based adaptations \citep{finn2017model}, learning loss functions \citep{houthooft2018evolved,oh2020discovering,xu2020meta}, probabilistic inference of context variables \citep{rakelly2019efficient,zintgraf2019varibad,fakoor2019meta}, online adaptation of hyper-parameters within a single lifetime \citep{xu2018meta,zahavy2020self} and so on. Some of these formulations have mathematical connections; see e.g. \citep{ortega2019meta} for an in-depth discussion.

We focus on gradient-based adaptations \citep{finn2017model}, where the agent carries out policy gradient updates \citep{sutton1999} at both meta-training and meta-testing time. Conceptually, since meta-RL seeks to optimize the way in which the agent adapts itself in face of new environments, it needs to \emph{differentiate} through the policy gradient update itself. This effectively reduces the meta-RL problem into estimations of Hessian matrices of value functions. 

\paragraph{Challenges for computing  Hessian matrix of value functions.} To calculate Hessian matrices (or unbiased estimates thereof) for supervised learning objectives, it suffices to differentiate through gradient estimates, which can be easily implemented with auto-differentiation packages \citep{asadi2017,jax2018github,paszke2017automatic}. However, it is not the case for value functions in RL. Intuitively, this is because value functions are defined via expectations with respect to distributions that themselves depend on the policy parameters of interest, whereas in supervised learning the expectations are defined with respect to a fixed data distribution. As a result, implementations that do not take this into account may lead to  estimates \citep{finn2017model} whose bias is not properly characterized and might have a negative impact on downstream applications~\citep{liu2019taming}. 

Motivated by this observation, a number of prior works suggest implementation alternatives that lead to unbiased Hessian estimates \citep{foerster2018dice}, with potential variance reduction \citep{farquhar2019loaded,mao2019baseline,liu2019taming}; or biased estimates with small variance \citep{rothfuss2018promp}. However, different algorithms in this space are motivated and derived in seemingly unrelated ways: for example,  \citep{foerster2018dice,farquhar2019loaded,mao2019baseline} derive code-level implementations within the general context of stochastic computation graphs \citep{schulman2015gradient}. On the other hand, \citep{rothfuss2018promp} derives the estimates by explicitly analyzing certain terms with potentially high variance, which naturally produces bias in the final estimate. Due to apparently distinct ways of deriving estimates, it is not immediately clear how  all such methods are related, and whether there could be other alternatives.

\begin{table*}[t]
\caption{Interpretation of prior work on higher-order derivative estimations as special instances of differentiating off-policy evaluation estimates. Given any off-policy evaluation estimate $\hat{V}^{\pi_\theta}$ from the second row, we recover higher-order derivative estimates in prior work in the first row, by differentiating through the estimate $\nabla_\theta^K \hat{V}^{\pi_\theta}$. 
}
\centering
\begin{small}
\begin{sc}
\begin{tabular}{C{0.6in} *4{C{1.in}}}\toprule[1.5pt]
\bf Off-policy evaluation estimates & Step-wise IS \citep{precup2001off} & Doubly-robust \citep{jiang2016doubly} & TayPO-$1$ \citep{kakade2002approximately,tang2020taylor} &  TayPO-$2$ \citep{tang2020taylor} \\\midrule
\bf Prior work & \bf DiCE \citep{foerster2018dice}  & \bf Loaded DiCE \citep{farquhar2019loaded,mao2019baseline} & \bf LVC \citep{rothfuss2018promp} & \bf Second-order (this work) \\
\bottomrule
\end{tabular}
\end{sc}
\end{small}
\vskip -0.1in
\label{table:smallsize}
\end{table*}

\paragraph{Central contribution.} 
We present a unified framework for estimating higher-order derivatives of value functions, based on the concept of off-policy evaluation. The main insights are summarized in Table~\ref{table:smallsize}, where most aforementioned prior work can be interpreted as special cases of our framework. Our framework has a few advantages: (1) it conceptually unifies a few seemingly unrelated prior methods; (2) it elucidates the bias and variance trade-off of the estimates; (3) it naturally produces new methods based on Taylor expansions of value functions \citep{tang2020taylor}. 

After a brief background introduction on meta-RL in Section~\ref{sec:background}, we will discuss the above aspects in detail in Section~\ref{s:hess}. From an implementation perspective, we will show in Section~\ref{sec:code} that both the general framework and the new method can be conveniently implemented in auto-differentiation libraries \citep{abadi2016tensorflow,paszke2017automatic}, making it amenable in a practical setup. Finally in Section~\ref{sec:exp}, we validate important claims based on this framework with experimental insights.

\section{Background}
\label{sec:background}
We begin with the notation and background on RL and meta-RL.

\subsection{Task-based reinforcement learning}

Consider a Markov decision process (MDP) with state space $\mathcal{X}$ and action space $\mathcal{A}$. At time $t \geq 0$, the agent takes action $a_t\in\mathcal{A}$ in state $x_t\in\mathcal{X}$, receives a reward $r_t$ and transitions to a next state $x_{t+1}\sim p(\cdot|x_t,a_t)$. Without loss of generality, we assume a single starting state $x_0$. Here, we assume the reward $r_t=r(x_t,a_t,g)$ to be a deterministic function of state-action pair $(x_t,a_t)$ and the task variable $g \in \mathcal{G}$. The task variable $g\sim p_\mathcal{G}$ is resampled for every episode. For example, $x_t\in\mathcal{X}$ is the sensory space of a running robot, $a_t\in\mathcal{A}$ 
is the control, $g$ is the episodic target direction in which to run and $r(x_t,a_t,g)$ is the speed in direction $g$. A policy $\pi:\mathcal{X}\rightarrow \mathcal{P}(\mathcal{A})$ specifies a distribution over actions at each state. For convenience, we define the value function $V^\pi(x,g)\coloneqq\mathbb{E}_\pi\left[\sum_{t=0}^\infty \gamma^t r(x_t,a_t,g) \; \middle| \;x_0=x \right]$ and Q-function $Q^\pi(x,a,g)\coloneqq\mathbb{E}_\pi\left[\sum_{t=0}^\infty \gamma^t r(x_t,a_t,g)\; \middle| \;x_0=x ,a_0=a\right]$. Without loss of generality, we assume the  policy to be smoothly parameterized as $\pi_\theta$ with parameter $\theta\in\mathbb{R}^{D}$. We further assume that the MDPs terminate within a finite horizon of $H<\infty$ under any policy. 

\subsection{Gradient-based meta-reinforcement learning}

The motivation of meta-RL is to identify a policy $\pi_\theta$ such that given a task $g$, after updating the parameter from $\theta$ to $\theta'$ with a parameter update computed under rewards $r(x,a,g)$, the resulting policy $\pi_{\theta'}$ performs well. Formally, define $U(\theta,g) \in \mathbb{R}^{D}$ as a parameter update to $\theta$, for example, one policy gradient ascent step. Model-agnostic meta-learning (MAML) \citep{finn2017model} formulates meta-RL as optimizing the value function at the updated policy $\mathbb{E}_{g\sim p_\mathcal{G}}\left[V^{\pi_{\theta'}}(x,g)\right]$ where $\theta'=\theta+U(\theta,g)$ is the updated policy. The optimization is with respect to the initial policy parameter $\theta$. The aim is to find $\theta$ such that it entails fast learning (or adaptation) to the environment, through the \emph{inner loop} update operator $U(\theta,g)$, that leads to high-performing updated policy $\theta'$. Consider the following problem,
\begin{align}
    \max_\theta\  F(\theta) \coloneqq   \mathbb{E}_{g\sim p_\mathcal{G}} \left[V^{\pi_{\theta'}}(x_0,g)\right], \ \theta'=\theta+U(\theta,g).\label{eq:maml}
\end{align}
We can decompose the meta-gradient into two terms based on the chain rule, 
\begin{align}
    \nabla_\theta F(\theta) = \mathbb{E}_{g\sim p_\mathcal{G}}\left[ \nabla_\theta F(\theta,g)\right] \coloneqq \mathbb{E}_{g\sim p_\mathcal{G}}\left[ \nabla_\theta \theta'  \nabla_{\theta'}   V^{\pi_{\theta'}}(x_0,g)\right] \label{eq:maml-hvp}
\end{align}
where $\nabla_\theta \theta'=I+\nabla_\theta U(\theta,g) \in \mathbb{R}^{D\times D}$ is a matrix, and $ \nabla_{\theta'}V^{\pi_{\theta'}}(x_0,g)$ is the vanilla policy gradient evaluated at the updated parameter $\theta'$ \citep{sutton1998}. A straightforward way to optimizing Eqn~\eqref{eq:maml} is to carry out the \emph{outer loop} update $\theta\leftarrow\theta +\alpha \nabla_\theta F(\theta)$ with learning rate $\alpha>0$.

\paragraph{Policy gradient update.} Following MAML \citep{finn2017model},
we focus on policy gradient update where $U(\theta,g)=\eta \nabla_\theta  V^{\pi_\theta}(x_0,g)$ for a fixed step size $\eta>0$. The matrix $
    \nabla_{\theta} U(\theta,g) = \eta \nabla_\theta^2  V^{\pi_\theta}(x_0,g)   
$ is the Hessian of the value function with respect to policy parameters; henceforth, we define $H_\theta(x_0,g)\coloneqq \nabla_\theta^2 V^{\pi_\theta}(x_0,g)\in\mathbb{R}^{D\times D}$.

\paragraph{Estimating meta-gradients.} Practical algorithms construct stochastic estimates of the meta-gradients using a finite number of samples. In Eqn~\eqref{eq:maml-hvp}, we decompose the meta-gradients into the multiplication of a Hessian matrix with a policy gradient vector.
Given a fixed task variable $g$, a common practice of prior work is to construct the gradient estimate as the product of the Hessian estimate and the policy gradient estimate $\hat{\nabla}F(\theta,g)=(I+\eta\hat{H}_\theta(x_0,g))\hat{\nabla}_{\theta'} V^{\pi_{\theta'
}}(x_0,g)$;
see, e.g., \citep{finn2017model,al2017continuous,stadie2018some,rothfuss2018promp,foerster2018dice,farquhar2019loaded,mao2019baseline,liu2019taming}. See Algorithm 1 for the full pseudocode for estimating meta-gradients by sampling multiple tasks. Since there is a large literature on constructing accurate estimates to policy gradient (e.g., actor-critic algorithms use baselines for variance reduction \citep{konda2000actor}), the main challenge consists in estimating the Hessian matrix accurately. 

\paragraph{Bias of common plug-in estimators.} 
Common practices in meta-RL algorithms rely on the premise that if both Hessian and policy gradient estimates are  unbiased, then the meta-gradient estimate is unbiased too.
\begin{align*}
    \mathbb{E}\left[\hat{H}_\theta(x_0,g)\right]=H_\theta(x_0,g), \mathbb{E}\left[\hat{\nabla}_{\theta'}V^{\pi_{\theta'}}(x_0,g)\right]=\nabla_{\theta'}V^{\pi_{\theta'}}(x_0,g) \Rightarrow \mathbb{E}\left[\hat{\nabla}_\theta F(\theta,g)\right] = \nabla_\theta F(\theta,g).
\end{align*} 
 Unfortunately, this is not true. This is because the two estimates are in general correlated when the sample size is finite, leading to the bias of the overall estimate. We provide further discussions in Appendix~\ref{appendix:mamlbias}. For the rest of the paper, we follow practices of prior work and focus on the properties of Hessian estimates, leaving a more proper treatment of this bias to future work. 

\begin{algorithm}[h]
\label{algo:meta-rl-estimate}
\begin{algorithmic}
\FOR{i=1,2...n}
\STATE Sample task variable $g_i\sim p_\mathcal{G}$.
\STATE Sample $B$ trajectories under policy $\pi_\theta$; Compute $B$-trajectory policy gradient estimate $\hat{\nabla}_\theta V^{\pi_\theta}(x_0,g_i)$, update parameter $\theta'=\theta+\hat{\nabla}_\theta V^{\pi_\theta}(x_0,g_i)$.
\STATE Compute $B$-trajectory Hessian estimate $\hat{H}_\theta(x_0,g_i)$ and an unbiased policy gradient estimate at $\theta'$, i.e.,  $\hat{\nabla}_{\theta'}V^{\pi_\theta}(x_0,g)$.
\STATE Compute the $i$-th meta-gradient estimate $\hat{\nabla}F(\theta,g_i) = \left(I+\eta\hat{H}_\theta(x_0,g_i)\right) \hat{\nabla}_{\theta'}V^{\pi_{\theta'}}(x_0,g_i)$.
\ENDFOR
\STATE Output averaged meta-gradient estimate $\frac{1}{n}\sum_{i=1}^n \hat{\nabla}_\theta F(\theta,g_i)$.
\caption{Pseudocode for computing meta-gradients for the MAML objective}
\end{algorithmic}
\end{algorithm}

\section{Deriving Hessian estimates with off-policy evaluation}
\label{s:hess}

Since the meta-gradient estimates are computed by averaging over task variables $g\sim p_\mathcal{G}$, 
 in the following, we focus on Hessian estimates at a single state and task variable $H_\theta(x,g)$ with a fixed $g$. In this section, we also drop the dependency of the value function on $g$, such that, e.g., $V^{\pi_\theta}(x_0)\equiv V^{\pi_\theta}(x_0,g)$ and $Q(x,a,g)\equiv Q(x,a)$.

\subsection{Off-policy evaluation: maintaining higher-order dependencies on parameters}
We assume access to data $(x_t,a_t,r_t)_{t=0}^\infty$ generated under a behavior policy $\mu$. Off-policy evaluation \citep{precup2001off} consists in building estimators $\hat{V}^{\pi_\theta}(x,g)$ using the behavior data such that
$
    \hat{V}^{\pi_\theta}(x)  \approx V^{\pi_\theta}(x)
$ for a range of target policies $\pi_\theta$. Note that the estimate  $\hat{V}^{\pi_\theta}(x)$ is a random variable depending on $(x_t,a_t,r_t)_{t=0}^\infty$, it is also a function of $\theta$. The approximation $\hat{V}^{\pi_\theta}(x)  \approx V^{\pi_\theta}(x)$ implies that $\hat{V}^{\pi_\theta}(x)$ is indicative of how the value function $V^{\pi_\theta}(x)$ depends on $\theta$, and hence maintains the higher-order dependencies on $\theta$. Throughout, we assume  $\pi_\theta(a|x)>0,\mu(a|x)>0$ for all $(x,a) \in \mathcal{X} \times \mathcal{A}$.

\paragraph{Example: step-wise importance sampling (IS) estimate.} As a concrete example, consider the unbiased step-wise IS estimate $\hat{V}_\text{IS}^{\pi_\theta}(x_0)=\sum_{t=0}^\infty \gamma^t \left(\Pi_{s\leq t}\rho_s^\theta\right) r_t$ where $\rho_s^\theta\coloneqq \pi_\theta(a_s|x_s)/\mu(a_s|x_s)$. Since the value function $V^{\pi_\theta}(x)$ is in general a highly non-linear function of $\theta$ (see discussions in \citep{agarwal2020optimality}) we see that $\hat{V}_\text{IS}^{\pi_\theta}(x_0)$ retains such dependencies via the sum of product of IS ratios.

\subsection{Warming up: deriving unbiased estimates with variance reduction}

We start with a general result based on the intuition above: given an estimate $\hat{V}^{\pi_\theta}(x)$ to
$V^{\pi_\theta}(x)$, we can directly use the $m$\textsuperscript{th}-order derivative $\nabla_\theta^m \hat{V}^{\pi_\theta}(x)\in \mathbb{R}^{D^K}$ as an estimate to $\nabla_\theta^K V^{\pi_\theta}(x)$. We introduce two assumptions: \textbf{(A.1)} $\hat{V}^{\pi_\theta}(x)$ is $m$\textsuperscript{th}-order differentiable w.r.t.\,$\theta$ almost surely.
    \textbf{(A.2)}  $\left\lVert \nabla_\theta^m \hat{V}^{\pi_\theta}(x)\right\rVert_\infty<M$ for some constant $M$ for the order $m$ of interest.
These assumptions are fairly mild; see further details in Appendix~\ref{appendix:assumption}. 
The following result applies to general unbiased off-policy evaluation estimates.

\begin{restatable}{proposition}{offpolicyestimate}\label{prop:offpolicyestimate} Assume \textbf{(A.1)} and \textbf{(A.2)} are satisfied. Further assume we have an estimator $\hat{V}^{\pi_\theta}(x)$ which is unbiased ($\mathbb{E}_\mu\left[\hat{V}^{\pi_{\theta'}}(x)\right]=V^{\pi_{\theta'}}(x)$) for all $\theta' \in N(\theta)$ where $N(\theta)$ is some open set that contains $\theta$. Under some additional mild conditions, the $m$\textsuperscript{th}-order derivative of the estimate $\nabla_\theta^m \hat{V}^{\pi_\theta}(x_0)$ are unbiased estimates to the $m$\textsuperscript{th}-order derivative of the value function $\mathbb{E}_\mu\left[\nabla_\theta^m \hat{V}^{\pi_\theta}(x)\right]=\nabla_\theta^m V^{\pi_\theta}(x)$ for $m\geq 1$. 
\end{restatable}

\paragraph{Doubly-robust estimates.}
As a special case, we describe the doubly-robust (DR) off-policy evaluation estimator \citep{dudik2014doubly,jiang2016doubly,thomas2016data}. Assume we have access to a state-action dependent critic $Q(x,a,g)$, and we use the notation $Q(x,\pi(x))\coloneqq \sum_{a} Q(x,a)\pi(a|x)$. The DR estimate is defined recursively as follows,
\begin{align}
    \hat{V}^{\pi_\theta}_\text{DR}(x_t) = Q(x_t,\pi_\theta(x_t)) +   \rho_t^\theta \left(r_t + \gamma  \hat{V}^{\pi_\theta}_\text{DR}(x_{t+1})  -Q(x_t,a_t)\right)
    \label{eq:db-estimate}.
\end{align}

The DR estimate is unbiased for all $\pi_\theta$ and subsumes the step-wise IS estimate as a special case when $Q(x,a)\equiv 0$. If the critic is properly chosen, e.g., $Q(x,a)\approx Q^{\pi_\theta}(x,a)$, it can lead to significant variance reduction compared to $\hat{V}_\text{IS}^{\pi_\theta}(x_0)$. By directly differentiating the estimate $\nabla_\theta^m \hat{V}_\text{DR}^{\pi_\theta}(x)$, we derive estimators for higher-order derivatives of the value function; the result for the gradient in Proposition~\ref{prop:offpolicyestimategrad} was shown in \citep{huang2020importance}.

\begin{restatable}{proposition}{offpolicyestimategrad}\label{prop:offpolicyestimategrad} Define $\pi_t\coloneqq \pi_\theta(a_t|x_t)$ and let $\delta_t\coloneqq r_t + \gamma \hat{V}_\text{DR}^{\pi_\theta}(x_{t+1})-Q(x_t,a_t)$ be the sampled temporal difference error at time $t$. Note that $\nabla_\theta \log \pi_t \in\mathbb{R}^{D}$ and $\nabla_\theta^2 \log \pi_t \in\mathbb{R}^{D\times D}$. The estimates of higher-order derivatives can be deduced recursively, and in particular for $m=1,2$,
\begin{align}
    \nabla_\theta \hat{V}_\text{DR}^{\pi_\theta}(x_t) &= \nabla_\theta Q(x_t,\pi_\theta(x_t)) + \rho_t^\theta \delta_t \nabla_\theta \log \pi_t   + \gamma \rho_t^\theta \nabla_\theta \hat{V}_\text{DR}^{\pi_\theta}(x_{t+1}),  \label{eq:recursive-pg}\\
    \nabla_\theta^2 \hat{V}_\text{DR}^{\pi_\theta}(x_t) &=  \rho_t^\theta \delta_t \left(\nabla_\theta^2 \log \pi_t + \nabla_\theta \log \pi_t \nabla_\theta \log \pi_t^T \right) + \gamma \rho_t^\theta \nabla_\theta \hat{V}_\text{DR}^{\pi_\theta}(x_t) \nabla_\theta \log \pi_t^T \nonumber \\
    &\ \ + \gamma \rho_t^\theta \nabla_\theta \log  \pi_t \nabla_\theta \hat{V}_\text{DR}^{\pi_\theta}(x_t)^T + \nabla_\theta^2 Q(x_t,\pi_\theta(x_t)) + \gamma \rho_t^\theta \nabla_\theta^2 \hat{V}_\text{DR}^{\pi_\theta}(x_{t+1}). \label{eq:recursive-hessian}
\end{align}
\end{restatable}
\paragraph{Bias and variance of Hessian estimates.}
Proposition~\ref{prop:offpolicyestimate} implies that $\nabla_\theta \hat{V}_{\pi_\theta}(x_0)$ and
$\nabla_\theta^2 \hat{V}_{\pi_\theta}(x_0)$ are both unbiased. To analyze the variance, we start with $m=1$: note that when on-policy $\mu=\pi_t$ \citep{sutton1999}, $\nabla_\theta V_\text{DR}^{\pi_\theta}(x_0)$ recovers a form of gradient estimates similar to actor-critic policy gradient with action-dependent baselines \citep{liu2017action,wu2018variance,tucker2018mirage}; when $Q(x,a)$ is only state dependent, $\nabla_\theta \hat{V}_\text{DR}^{\pi_\theta}(x)$ recovers the common policy gradient estimate with state-dependent baselines \citep{konda2000actor}. As such, the estimates are computed with potential variance reduction due to the critic.
Previously, \citep{huang2020importance} started with the DR estimate and derived a more general result for the on-policy first-order case. For the Hessian estimate, we expect a similar effect of variance reduction as shown in  experiments.

\paragraph{Recovering prior work on estimates to higher-order derivatives.} When applied to meta-RL, DiCE \citep{foerster2018dice} and its follow-up variants \citep{farquhar2019loaded,mania2018simple} can be seen as special cases of $\nabla_\theta^2 \hat{V}_\text{DR}^{\pi_\theta}(x)$ with different choices of the critic $Q$ when evaluated on-policy $\mu=\pi_\theta$. See Table 1 for the correspondence between prior work and their equivalent formulations under the framework of off-policy evaluation. We will discuss  detailed pseudocode in Section~\ref{sec:code}. See also Appendix~\ref{appendix:code} for more details.

\subsection{Trading-off bias and variance with Taylor expansions}

Starting from unbiased off-policy evaluation estimates $\hat{V}^{\pi_\theta}(x)$, we can directly construct unbiased estimates to higher-order derivatives by differentiating the original estimate to obtain  $\nabla_\theta^m \hat{V}^{\pi_\theta}(x)$. However, unbiased estimates can have large variance. Though it is possible to reduce variance through the critic, as we will show experimentally, this is not enough to counter the high variance due to products of IS ratios. This leads us to consider trading off bias with variance \citep{rowland2019adaptive}.
 
Since we postulate that the products of IS ratios lead to high variance, we might seek to control the number of IS ratios in the estimate. We briefly introduce Taylor expansion policy optimization (TayPO) \citep{tang2020taylor}, a natural framework to control for the number of IS ratios in the value estimate.

\paragraph{Taylor expansions of value functions.} Consider the value function $V^{\pi_\theta}(x_0)$ as a function of $\pi_\theta$. Using the idea of Taylor expansions, we can express  $V^{\pi_\theta}(x_0)$ as a sum of polynomials of $\pi_\theta-\mu$. We start by defining the $K$\textsuperscript{th}-order increment as $U_0^{\pi_\theta}(x_0)=V^\mu(x_0)$, which does not contain any IS ratio (zeroth order); and for $K\geq1$,
\begin{align}
    U_K^{\pi_\theta}(x_0) \coloneqq \mathbb{E}_{\mu} \Bigg[ \underbrace{ \sum_{t_1=0}^\infty \sum_{t_2=t_1+1}^\infty ...\sum_{t_K=t_{K-1}+1}^\infty \gamma^{t_K} \left(\Pi_{i=1}^K (\rho_{t_i}^\theta-1)\right) Q^\mu(x_{t_K},a_{t_K})}_{\hat{U}_K^{\pi_\theta}(x_0)} \Bigg].\label{eq:increment}
\end{align}
Intuitively, the $K$\textsuperscript{th}-order increment only contains product of $K$ IS ratios. Equation~\eqref{eq:increment} also yields a natural sample-based estimate $\hat{U}_K^{\pi_\theta}(x_0)$, which we will discuss later. The $K$\textsuperscript{th}-order Taylor expansion is defined as the partial sum of increments 
$V_{K}^{\pi_\theta}(x_0)\coloneqq \sum_{k=0}^K U_k^{\pi_\theta}(x_0) $.
Since $V_{K}^{\pi_\theta}(x_0)$ consists of products of up to $K$ IS ratios, it is effectively the $K$\textsuperscript{th}-order Taylor expansion of the value function. The properties are summarized as follows. 

\begin{restatable}{proposition}{taypo}\label{prop:taypo}(Adapted from Theorem 2 of \citep{tang2020taylor}.)
 Define $\left\lVert \pi - \mu \right\rVert_1\coloneqq \max_{x}\sum_{a}|\pi(a|x)-\mu(a|x)|$. Let $C$ be a constant and $\epsilon=\frac{1-\gamma}{\gamma}$. Then the following holds for all $K\geq 0 $,
\begin{align}
    V^{\pi_\theta}(x_0,g) = \underbrace{ V_{K}^{\pi_\theta}(x_0,g)}_{K\text{-th order expansion}}+ \underbrace{C(\left\lVert \pi_\theta - \mu \right\rVert_1 / \epsilon)^{K+1}}_{\text{residual}} ,\label{eq:taypo-decomposition}
\end{align}
 If $\left\lVert \pi_\theta - \mu \right\rVert_1<\epsilon$, then $V^{\pi_\theta}(x_0) = 
 \lim_{K\rightarrow\infty}V_{K}^{\pi_\theta}(x_0) = \mathbb{E}_\mu \left[  \sum_{k=0}^\infty U_k^{\pi_\theta}(x_0) \right]$.
\end{restatable}

\paragraph{Sample-based estimates to Taylor expansions of value functions.} As shown in Equation~\eqref{eq:increment}, $\hat{U}_K^{\pi_\theta}(x_0)$ is an unbiased estimate to $U_K^{\pi_\theta}(x_0)$. We can naturally define the sample-based estimate to the $K$\textsuperscript{th}-order Taylor expansion, called the TayPO-$K$ estimate,
\begin{align}
    \hat{V}_{K}^{\pi_\theta}(x_0) \coloneqq \sum_{k=0}^K \hat{U}_k^{\pi_\theta}(x_0).
    \label{eq:taypo}
\end{align}
The expression of $\hat{U}^{\pi_\theta}(x_0)$ contains $O(T^K)$ terms if the trajectory is of length $T$. Please refer to Appendix~\ref{appendix:taypo} for further details on computing the estimates in linear time $O(T)$ with sub-sampling. Note that  $\hat{V}_{K}^{\pi_\theta}(x_0)$ is a sample-based estimate whose bias against the value function is controlled by the residual term which decays exponentially when $\pi_\theta$ and $\mu$ are close. Similar to how we derived the unbiased estimate $\nabla_\theta^m \hat{V}_\text{DR}^{\pi_\theta}(x)$, we can differentiate through the TayPO-$K$ value estimate to produce estimates to higher-order derivatives $\nabla_\theta^m \hat{V}_{K}^{\pi_\theta}(x_0)$. 

\paragraph{Bias and variance of Hessian estimates.} TayPO-$K$ trades-off bias and variance with choices of $K$. To understand the variance, note that TayPO-$K$  limits the number multiplicative IS ratios to be $K$. Though it is difficult to compute the variance, we argue that the variance generally increases with $K$ as the number of IS ratios increase \citep{precup2001off,jiang2016doubly,munos2016safe}. We characterize the bias of TayPO-$K$ as follows.

\begin{restatable}{proposition}{taypogradonpolicy}\label{prop:taypogradonpolicy}
Assume \textbf{(A.1)} and \textbf{(A.2)} hold. Also assume $\left\lVert \pi_\theta-\mu\right\rVert_1 \leq \epsilon=(1-\gamma)/\gamma$. For any tensor $x$, define $\left\lVert x\right\rVert_\infty \coloneqq \max_i \left|x[i]\right|$.
The $K$\textsuperscript{th}-order TayPO objective produces the following bias in estimating high-order derivatives,
\begin{align}
  \left\lVert \mathbb{E}_\mu\left[\nabla_\theta^m \hat{V}_K^{\pi_\theta}\right](x_0) - \nabla_\theta^m V^{\pi_\theta}(x_0)\right\rVert_\infty \leq \sum_{k=K+1}^\infty  \left\lVert \nabla_\theta^m U_k^{\pi_\theta}(x_0) \right\rVert_\infty .\label{eq:taypo-grad}
\end{align}
Hence the upper bound for the bias decreases as $K$ increases.
Importantly, when on-policy $\mu=\pi_\theta$, the $K$\textsuperscript{th}-order TayPO objective preserves up to $K$\textsuperscript{th}-order derivatives for any $K\geq 0$,
\begin{align}
   \mathbb{E}_\mu\left[\nabla_\theta^m \hat{V}_{K}^{\pi_\theta}(x_0) \right]= \nabla_\theta^m V^{\pi_\theta}(x_0) ,\forall m\leq K .\label{eq:taypo-grad-onpolicy}
\end{align}
\end{restatable}

Though IS ratios $\rho_t^\theta$ evaluate to $1$ when on-policy, they maintain the parameter dependencies  in differentiations. As such, higher-order expansions contains products of IS ratios of higher orders, and maintains the high-order dependencies on parameters more accurately.  There is a clear trade-off between bias and variance mediated by $K$. When $K$ increases, the higher-order derivatives are maintained more accurately in expectation, leading to less bias. However, the variance increases too.

\paragraph{Recovering prior work as special cases.} Recently, \citep{rothfuss2018promp} proposed a low variance curvature (LVC) Hessian estimate. This estimate is equivalent to the TayPO-$K$ estimate with $K=1$. As also noted by \citep{rothfuss2018promp}, their objective function bears similarities to first-order policy search algorithms  \citep{kakade2002approximately,schulman2015,schulman2017}, which have in fact been  interpreted as  first-order special cases of $\nabla_\theta\hat{V}_{K}^{\pi_\theta}(x) $  with $K=1$ \citep{tang2020taylor}. Importantly, based on Proposition~\ref{prop:taypogradonpolicy}, the LVC estimate only maintains the first-order dependency perfectly but introduces bias when approximating the Hessian, even when on-policy.

\paragraph{Limitations.} Though the above framework interprets a large number of prior methods as special cases, it has some limitations. For example, the derivation of Hessian estimates based on the DR estimate (Proposition~\ref{prop:offpolicyestimategrad}) involves estimates of the value function $\hat{V}_\text{DR}^{\pi_\theta}(x_t)$. In practice, when near on-policy $\pi_\theta\approx\mu$, one might replace the DR estimate $\hat{V}_\text{DR}^{\pi_\theta} $ by other value function estimate, such as plain cumulative sum of returns or TD($\lambda$), $\hat{V}_{\text{TD}(\lambda)}^{\pi_\theta}$ \citep{farquhar2019loaded,mao2019baseline} in Eqn~\eqref{eq:recursive-pg}. As such, the practical implementation might not strictly adhere to the conceptual framework.
In addition, TMAML  \citep{liu2019taming} is not incorporated as part of this framework: we show in Appendix~\ref{appendix:tmaml} that the control variate introduced by TMAML in fact biases the overall estimate. 

\section{From Hessian estimates to meta-gradient estimates}\label{sec:code}

A practical desiteratum for meta-gradient estimates is that it can be implemented in a scalable way using auto-differentiation frameworks \citep{abadi2016tensorflow,paszke2017automatic}. Below, we discuss how this can be achieved.

\subsection{Auto-differentiating off-policy evaluation estimates for Hessian estimates}
\label{sec:code}
In practice, we seek Hessian estimates that could be implemented with the help of an established framework, such as auto-differentiation libraries \citep{abadi2016tensorflow,paszke2017automatic}. Now we discuss how to conveniently implement ideas discussed in the previous section.

\begin{algorithm}[h]
\label{algo:evaluation-subroutine}
\begin{algorithmic}
\REQUIRE \textbf{Inputs}: Trajectory $(x_t,a_t,r_t)_{t=0}^T$, target policy $\pi_\theta$, behavior policy $\mu$, (optional) critic $Q$.
\STATE Initialize $\hat{V}=Q(x_T,\pi_\theta(x_T),g)$.
\FOR{$t=T-1,\dots0$}
\STATE Compute IS ratio $\rho_t^\theta=\pi_\theta(a_t|x_t)/\mu(a_t|x_t)$.
\STATE Recursion: $\hat{V}\leftarrow Q(x_t,\pi_\theta(a_t),g) + \gamma \rho_t^\theta (r_t + \gamma Q(x_{t+1},\pi_\theta(x_{t+1}),g) - Q(x_t,a_t)) + \gamma \rho_t^\theta \hat{V}$.
\ENDFOR
\STATE Output $\hat{V}$ as an estimate to $V^{\pi_\theta}(x_0,g)$.
\caption{Example: an off-policy evaluation subroutine for computing the DR estimate}
\end{algorithmic}
\end{algorithm}

\paragraph{Auto-differentiating the estimates.} We can abstract the off-policy evaluation  as a function $\text{eval}(D,\theta)$ that takes in some data $D$ and parameter $\theta$, and outputs an estimate for $V^{\pi_\theta}(x_0,g)$. In particular, $D$ includes the trajectories and $\theta$ is input via the policy $\pi_\theta$. As an example, Algorithm 2 shows that for the doubly-robust estimator, the dependency of the estimator on $\theta$ is built through the recursive computation by $\text{eval}(D,\theta)$. In fact, if we implement Algorithm 2 with an auto-differentiation framework (for example, Tensorflow or PyTorch \citep{asadi2017,jax2018github,paszke2017automatic}), the higher-order dependency of $\hat{V}$ on $\theta$ is maintained through the computation graph. We can compute the Hessian estimate by directly differentiating through the function output
$
    \nabla_\theta^K \hat{V}=\nabla_\theta^K \text{eval}(D,\theta)
\approx \nabla_\theta^K V^{\pi_\theta}(x_0,g)$. In Appendix~\ref{appendix:code} we show how the estimates could be conveniently implemented, as in many deep RL agents (see, e.g., \citep{munos2016safe,wang2016,espeholt2018impala,kapturowski2018recurrent}). We also show concrete ways to compute estimates with TayPO-$K$ based on \citep{tang2020taylor}.

\subsection{Practical implementations of meta-gradient estimates}

In Equation~\eqref{eq:maml-hvp}, we write the meta-gradient estimate as a product between an Hessian estimate and a policy gradient. In practice, the meta-gradient estimate is computed via Hessian-vector products to avoid explicitly computing the Hessian estimate of size $D^2$. As such, the meta-gradient estimate could be computed by auto-differentiating through a scalar objective. See Appendix~\ref{appendix:code} for details. 

\paragraph{Bias and variance of meta-gradient estimates.} Intuitively, the bias and variance of the Hessian estimates translate into bias and variance of the downstream meta-gradient estimates. Prior work has showed that low variance of meta-gradient estimates lead to faster convergence \citep{fallah2020convergence,ji2020multi}. However, it is not clear how the bias (such as bias introduced by the Hessian estimates, or the bias due to correlated estimates) theoretically impacts the convergence. We will study empirically the effect of bias and variance in experiments, and leave further theoretical study to future work.

\section{Experiments}
\label{sec:exp}
We now carry out several empirical studies to complement the framework developed above. In Section~\ref{sec:insights}, we use a tabular example to investigate the bias and variance trade-offs of various estimates, to assess the validity of our theoretical insights. We choose a tabular example because it is straightforward to compute exact higher-order derivatives of value functions and make comparison. In Section~\ref{sec:balls} and Section~\ref{sec:more}, we apply the new second-order estimate in high-dimensional meta-RL experiments, to assess the potential performance gains in a more practical setup.
Though we can compute TayPO-$K$ order estimates for general $K$, we focus on $K\leq 2$ in experiments. Below, we also address TayPO-$1$ and TayPO-$2$ estimates as the first and second-order estimates respectively.

\subsection{Investigating the bias and variance trade-off of different estimates}
\label{sec:insights} 
We study the bias and variance trade-off of various estimates using a tabular exmaple. We consider random MDPs with $|\mathcal{X}|=10,|\mathcal{A}|=5$. The transition matrix of the MDPs are sampled from a Dirichlet distribution. See Appendix~\ref{appendix:exp} for further details. The policy $\pi_\theta$ is parameterized as $\pi_\theta(a|x)=\exp(\theta(x,a))/\sum_{b} \exp(\theta(x,b))$. The behavior policy $\mu$ is uniform and $\theta$ is initialized so that $\theta(x,a)=\log \pi(a|x)$ where $\pi = (1-\epsilon) \mu+\epsilon \pi_d$ for some deterministic policy $\pi_d$ and parameter $\epsilon\in [0,1]$. The hyper-parameter $\epsilon$ measures the off-policyness. In this example, there is no task variable. As performance metrics, we measure the component-wise correlation between the true derivatives $\nabla_\theta^K V^{\pi_\theta}(x_0)$ (computed via an oracle) where $x_0$ is a fixed starting state, and its estimates $\nabla_\theta^K \hat{V}^{\pi_\theta}(x_0)$, as commonly used in prior work \citep{foerster2018dice,farquhar2019loaded,mao2019baseline}. The estimates are averaged across $N$ samples. We study the effect of different choices of the off-policy estimate $\hat{V}^{\pi_\theta}$, as a function of off-policyness $\epsilon$ and sample size $N$. We report results with $\text{mean}\pm\text{standard error}$ over $10$ seeds.

\begin{figure}[t]
    \centering
    \subfigure[Gradient - off-policy ]{\includegraphics[keepaspectratio,width=.24\textwidth]{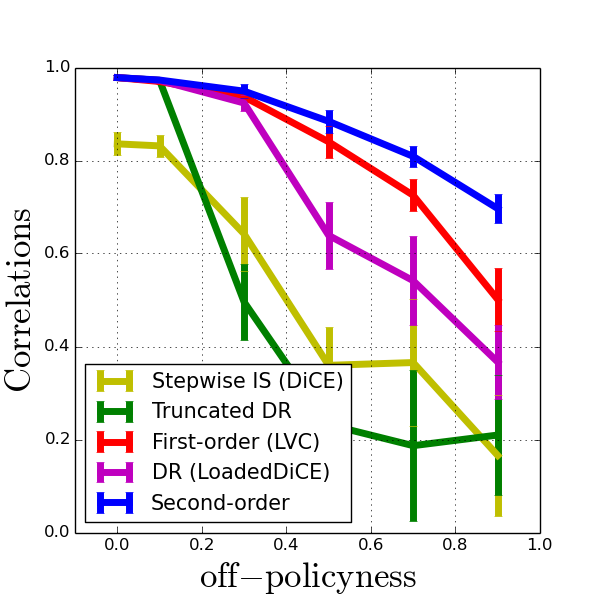}}
    \subfigure[Hessian - off-policy ]{\includegraphics[keepaspectratio,width=.24\textwidth]{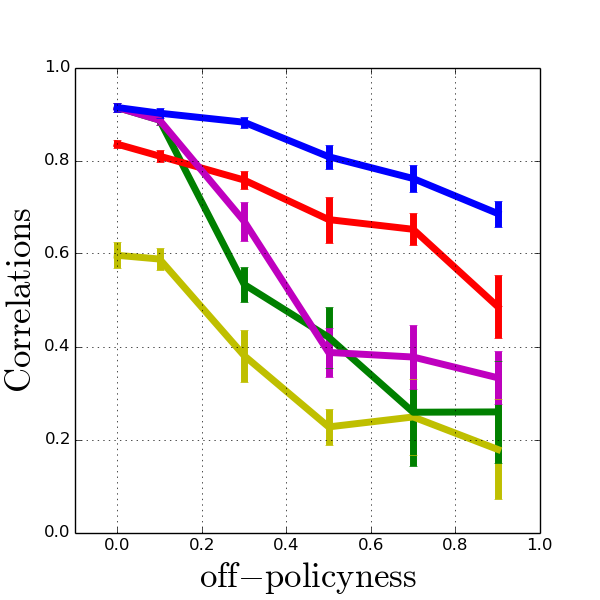}}
    \subfigure[Hessian - sample size ]{\includegraphics[keepaspectratio,width=.24\textwidth]{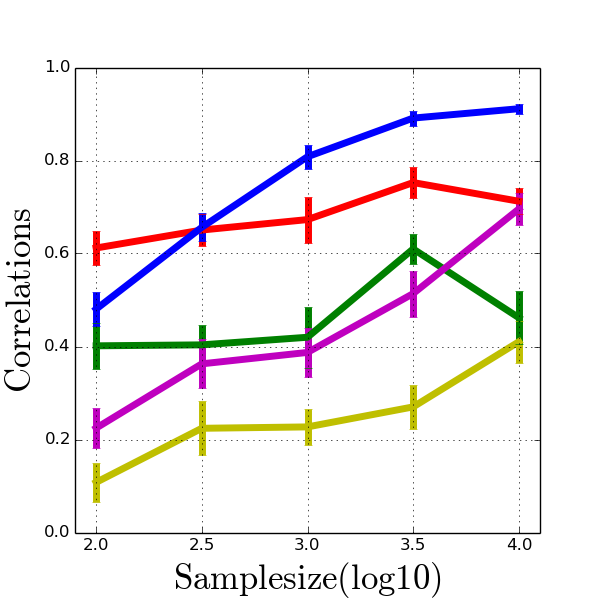}}
    \subfigure[2-$D$ control: training  ]{\includegraphics[keepaspectratio,width=.24\textwidth]{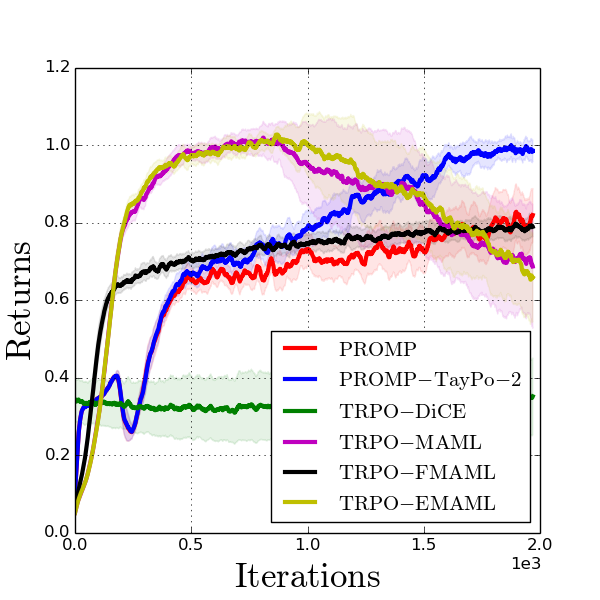}}
    \caption{Fig (a)-(b): performance measure as a function of off-policyness measured by $\epsilon=\left\lVert \pi_\theta-\mu \right\rVert_1$. (a) shows results for gradient estimation and (b) shows Hessians. Plots show the accuracy measure between the estimates and the ground truth. Overall, the second-order estimate achieves a better bias and variance trade-off. Fig (c): performance measure as a function of off-policyness measured by sample size $N$ for Hessians. Fig (d): training curves for the 2-$D$ control environment. The second-order estimate is generally more robust. All curves are averaged over $10$ runs. The left three plots use share the same legends.}
    \label{fig:tabular}
\end{figure}

\paragraph{Effect of off-policyness.}
In Figure~\ref{fig:tabular}(a) and  (b), we let $N=1000$ and show the performance as a function of $\epsilon$. When $\epsilon$ increases (more off-policy), the accuracy measures of most estimates decrease. This is because off-policyness generally increases both bias and variance (large IS ratios $\rho^\theta$). At this level of sample size, the performance of the second-order estimate degrades more slowly than other estimates, making it the dominating estimate across all values of $\epsilon$.
We also include truncated DR estimate using $\text{min}(\rho_t^\theta,\bar{\rho})$ as a baseline inspired by V-trace \citep{munos2016safe,espeholt2018impala}. The truncation is motivated by controlling the variance of the overall estimate. However, the estimate is heavily biased and does not perform well unless $\pi_\theta\approx \mu$. See Appendix~\ref{appendix:code} for more.

Consider the case when $\epsilon=0$ and the setup is on-policy. When estimating the policy gradient, almost all estimates converge to the optimal level of accuracy, except for the step-wise IS estimate, where the variance still renders the performance sub-optimal. However, when estimating the Hessian matrix, the first-order estimate converges to a lower accuracy than both second-order estimate and DR estimate. This validates our theoretical analysis, as both the second-order estimate and DR estimate are unbiased when on-policy.

\paragraph{Effect of sample size.} Figure~\ref{fig:tabular}(c) shows the accuracy measures as a function of sample size $N$ when fixing $\epsilon=0.5$. Note that since sample sizes directly influence the variance, the results show the variance properties of different estimates. When $N$ is small, the first-order estimate dominates due to smaller variance; however, when $N$ increases, the first-order estimate is surpassed by the second-order estimate, due to higher bias. For more results and ablation on high-dimensional environments, see Appendix~\ref{appendix:exp}.

\subsection{High-dimensional meta-RL problems}
Next, we study the practical gains entailed by the second-order estimate in high-dimensional meta-RL problems. We first introduce a few important algorithmic baselines, and how the second-order estimate is incorporated into a meta-RL algorithm.
\paragraph{Baseline algorithms.}
All baseline algorithms use plain stochastic gradient ascent as the inner loop optimizer: $\theta'=\theta+ \eta \nabla_\theta \hat{V}^{\pi_\theta}(x_0,g)$ where $g\sim p_\mathcal{G}$ is a sampled goal and $\hat{V}^{\pi_\theta}(x_0,g)$ is a  sample-based estimate of policy gradients averaged over $n$ trajectories. Different algorithms differ in how the inner loop loss is implemented, such that auto-differentiation produces different Hessian estimates:
these include TRPO-DiCE \citep{foerster2018dice,farquhar2019loaded,mao2019baseline}, TRPO-MAML \citep{finn2017model}, TRPO-FMAML \citep{finn2017model}, TRPO-EMAML \citep{al2017continuous,stadie2018some}. Please refer to Appendix~\ref{appendix:exp} for further details. Note despite the name, TRPO-DiCE baseline uses DR estimate to estimate Hessians. We implement the second-order estimate using the proximal meta policy search (PROMP) \citep{rothfuss2018promp} as the base algorithm. By default, PROMP uses the first-order estimate. Our new algorithm is named PROMP-TayPO-$2$.

\subsubsection{Continuous control in 2D environments}
\label{sec:balls}

\paragraph{Environment.} We consider a simple $2$-D navigation task introduced in \citep{rothfuss2018promp}. The state $x_t$ is the coordinate of a ball placed inside a room, the action $a_t$ is the direction in which to push the ball. The goal $g \in \mathbb{R}^4$ is an one-hot encoding of which corner of the room contains positive rewards. With $3$ adaptation steps, the agent should ideally navigate to the desired corner indicated by $g$.

\paragraph{Training performance.}  In Figure~\ref{fig:tabular}(d), we show the performance curves of various baseline algorithms. Though MAML and EMAML learns quickly during the initial phase of training, they ultimately become unstable. TRPO-DiCE generally underperforms other methods potentially due to bias. On the other hand, FMAML and PROMP both reduce variance at the cost of bias, but they both achieve a slightly lower level of performance compared to PROMP-TayPO-$2$. Overall, PROMP-TAyPO-$2$ achieves much more stable training curves compared to others, potentially owing to the better bias and variance trade-off in the Hessian estimates.

\subsubsection{Large scale locomotion experiments}
\label{sec:more}

\paragraph{Environments.} We consider the set of meta-RL tasks based on simulated locomotion in MuJoCo \citep{todorov2012}. Across these tasks, the states $x_t$ consist of robotic sensory inputs, the actions $a_t$ are torque controls applied to the robots. The task $g$ is defined per environment: for example, in random goal environment, $g\in\mathbb{R}^2$ is a random $2$-d goal location that the robot should aim to reach.

\paragraph{Experiment setup.} We adapt the open source code base by \citep{rothfuss2018promp} and adopt exactly the same experimental setup as \citep{rothfuss2018promp}. At each iteration, the agent samples $n=40$ task variables. For each task, the agent carries out $K=1$ adaptation computed based on $B=20$ trajectories sampled from the environment, each of length $T=100$. See Appendix~\ref{appendix:exp} for further details on the architecture and other hyper-parameters.  We report averaged results with $\text{mean}\pm\text{std}$ over $10$ seeds.

\begin{figure}[h]
    \centering
    \subfigure[\textbf{Ant goal}]{\includegraphics[keepaspectratio,width=.24\textwidth]{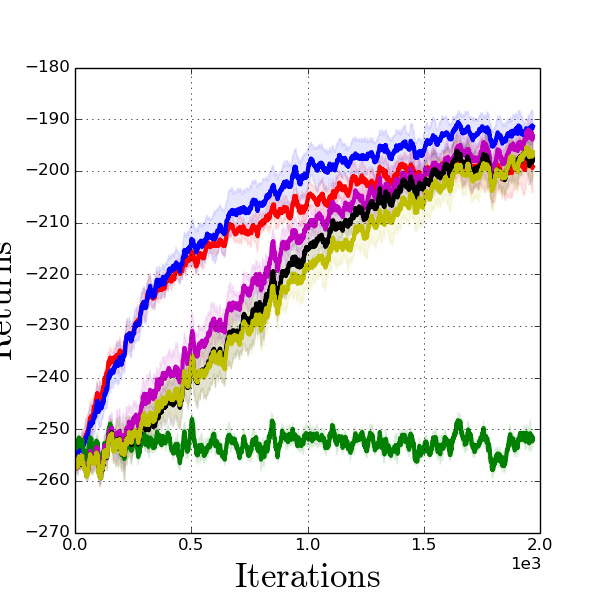}}
    \subfigure[\textbf{HalfCheetah }]{\includegraphics[keepaspectratio,width=.24\textwidth]{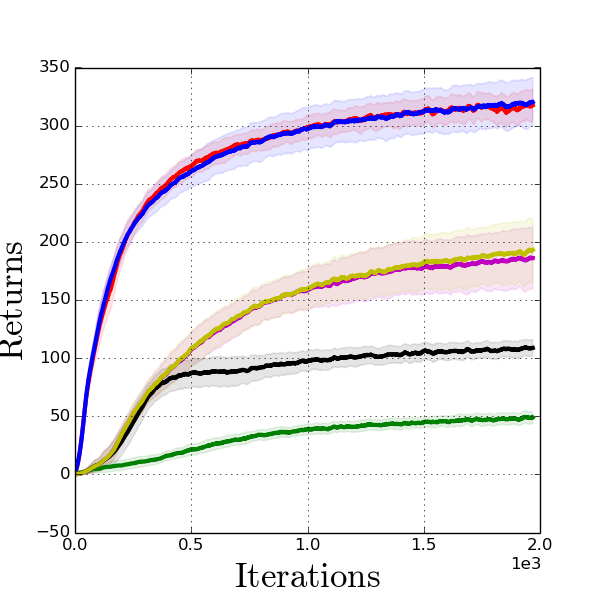}}
    \subfigure[\textbf{Ant }]{\includegraphics[keepaspectratio,width=.24\textwidth]{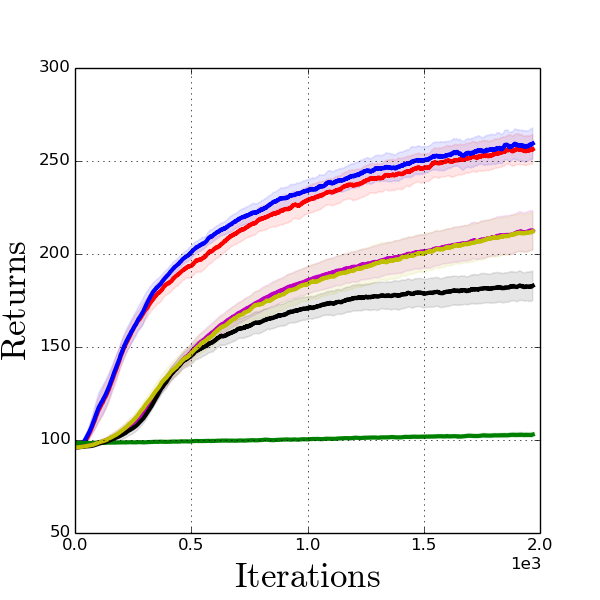}}
    \subfigure[\textbf{Walker2D }]{\includegraphics[keepaspectratio,width=.24\textwidth]{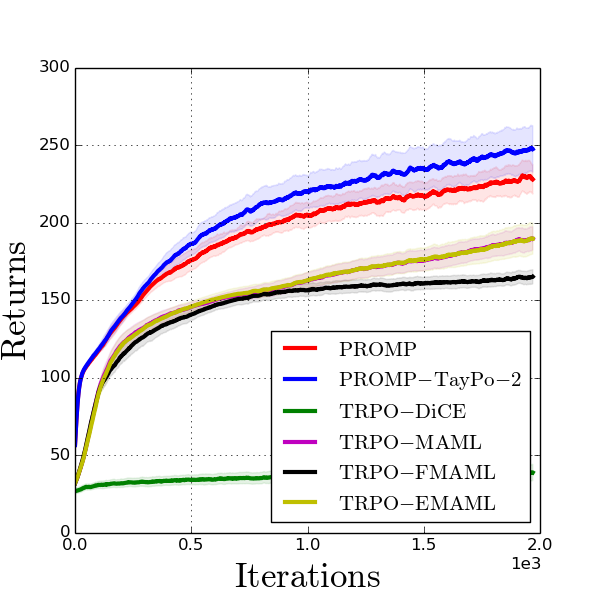}}
    \caption{Comparison of baselines over a range of simulated locomotion tasks. For task (b)-(d), the goal space consists of $2$-d random direction in which the robot should run to obtain positive rewards. For task (a), the goal space is a $2$-d location on the plane. Each curve shows the $\text{mean}\pm \text{std}$ across $5$ seeds. Overall, the second-order estimate achieves marginal gains over the first-order estimate.}
    \label{fig:meta}
\end{figure}

\paragraph{Results.} The training performance of different algorithmic baselines are shown in Figure~\ref{fig:meta}. Comparing TRPO-DiCE, TRPO-MAML and PROMP: we see that the results are compatible with those reported in \citep{rothfuss2018promp}, where PROMP outperforms TRPO-MAML, while TRPO-DiCE generally performs the worst potentially due to high variance in the gradient estimates. As a side observation, TRPO-FMAML generally underperforms TRPO-MAML, which implies the necessity of carrying out approximations to the Hessian matrix beyond the identity matrix. PROMP-TayPO-$2$ slightly outperforms PROMP in a few occasions, where the new algorithm achieves slightly faster learning speed and sometimes higher final performance. However, overall, we see that the empirical gains are marginal. This implies that under the default setup of these meta-RL experiments, the variance might be a major factor in gradient estimates, and the first-order estimate is near optimal compared to other estimates. See Appendix~\ref{appendix:exp} for additional experiments.

\section{Conclusion}
We have unified a number of important prior work on meta-gradient estimations for model-agnostic meta-RL. Our analysis entails the derivations of prior methods based on the unifying framework of differentiating through off-policy evaluation estimates. This framework provides a principled way to reason about the bias and variance in the higher-order derivative estimates of value functions, and opens the door to a new family of estimates based on novel off-policy evaluation estimates. As an important example, we have theoretically and empirically studied the properties of the family of TayPO-based estimates. It is worth noting that this framework further suggests any future advances in off-policy evaluation could be conveniently imported into potential improvements in meta-gradient estimates. As future work, we hope to see the applications of such principled estimates in broader meta-RL applications.

\paragraph{Acknowledgement.} The authors thank David Abel for reviewing an early draft of this work and providing very useful feedback. Yunhao and Tadashi are  thankful for the Scientific Computation
and Data Analysis section at the Okinawa Institute of Science and Technology (OIST), which maintains a cluster we
used for many of our experiments.

\bibliographystyle{unsrt}
\bibliography{your_bib_file}

\newpage
\appendix

\section{Bias with practical estimations of meta-gradients}
\label{appendix:mamlbias}

Recall the MAML objective defined in Eqn~\eqref{eq:maml}. The adapted parameter $\theta'=\theta+U(\theta,g)$ is computed with the expected policy gradient update $U(\theta,g)=\nabla_\theta V^{\pi_\theta}(x_0,g)$. This makes it difficult to construct fully unbiased estimates to the MAML gradient. We use a following example to show the intuitions.

 Consider a scalar objective $f(x)$ with input $x$. If it is possible to construct unbiased estimate to $x$, i.e., with $X$ such that $\mathbb{E}[X]=x$. It is difficult to construct unbiased estimates to $f(x)$ because $\mathbb{E}\left[f(X)\right]\neq f(x)$ unless $f$ is linear.

Conceptually, we can think of the argument $x$ here as the updated parameter resulting from expected updates $\theta'\coloneqq \theta+\eta\nabla_\theta V^{\pi_\theta}(x_0,g)$. It is convenient to construct unbiased estimate of this parameter because it is convenient to build unbiased estimate to the policy gradient $\nabla_\theta V^{\pi_\theta}(x_0,g)$. However, in our case, to compute the meta RL gradients, we need to evaluate the policy gradient at $\theta'$: $\nabla_{\theta'} V^{\pi_{\theta'}}(x_0,g)$, which is usually an highly non-linear function of $\theta'$. Using the notation of the above scalar objective example, we can construct a vector valued function: $f:\theta \mapsto \nabla_\theta V^{\pi_\theta}(x_0,g)$. Though it is straightforward to construct unbiased estimates $\hat{\theta'}$ to $\theta'$ with sampled trajectories, it is not easy to estimate $f(\theta')$ in an unbiased way, as a direct plug in $f(\hat{\theta'})$ would be biased.

\section{Proof on the bias of TMAML}\label{appendix:tmaml}
Below, we adopt the trajectory-based  notation of TMAML \citep{liu2019taming}. Let $\rho_t^\theta\coloneqq\frac{\pi_\theta(a_t|x_t)}{\text{sg}(\pi_\theta(a_t|x_t))}$, where the operation $\text{sg}(x)$ removes the dependency of $x$ on parameter $\theta$. In other words, $\nabla_\theta \text{sg}(x)=0$. It is worth noting that the stop gradient notations are equivalent to the derivation that defines $\rho_t^\theta=\frac{\pi_\theta(a_t|x_t)}{\mu(a_t|x_t)}$ with a fixed behavior policy $\mu$, and later sets $\mu=\pi_\theta$, as done in the main paper. TMAML \citep{liu2019taming} proposed the following baseline objective in the undiscounted finite horizon case with horizon $H<\infty$,
\begin{align*}
    J = \sum_{t=0}^{H-1}\left(1-\left(\Pi_{s=0}^t \rho_s^\theta\right)\right) (1-\rho_t^\theta) b(x_t),
\end{align*}
where $b(x_t)$ is a baseline function that depends on state $x_t$. The major claim from TMAML \citep{liu2019taming} is that 
$
    \mathbb{E}_{\pi_\theta}\left[\nabla_\theta^2 J\right] = 0$, which implies that adding $J$ to the original DICE objective \citep{foerster2018dice} might lead to variance reduction because it serves as a control variate term. However, below we show \begin{align*}
        \mathbb{E}_{\pi_\theta}\left[\nabla_\theta^2 J\right] \neq 0.
    \end{align*}

The proof proceeds in deriving the explicit expression for $\nabla_\theta^2 J$. We derive the same expression as that in the Appendix C of the original paper \citep{liu2019taming}, i.e.,
\begin{align}
    \nabla_\theta^2 J = 2\sum_{t=0}^{H-1} \nabla_\theta \log \pi_\theta(a_t|x_t) \left(\sum_{s=t}^{H-1} \nabla_\theta \log \pi_\theta (a_s|x_s) b(x_s)\right)^T.
    \label{eq:taming}
\end{align}
Following the proof of \citep{liu2019taming}, it is then straightforward to show that 
\begin{align}
    \mathbb{E}_{\pi_\theta}\left[ 2\sum_{t=0}^{H-1} \nabla_\theta \log \pi_\theta(a_t|x_t) \left(\sum_{s=t+1}^{H-1} \nabla_\theta \log \pi_\theta (a_s|x_s) b(x_s)\right)^T\right]=0.\label{eq:proof-taming}
\end{align}
However, note the difference between Eqn~\eqref{eq:taming} and Eqn~\eqref{eq:proof-taming} lies in the summation $s=t$ instead of $s=t+1$. Accounting for this difference, we have 
\begin{align*}
   \mathbb{E}_{\pi_\theta}\left[\nabla_\theta^2 J\right]  =\mathbb{E}_{\pi_\theta}\left[  2\sum_{t=0}^{H-1} \nabla_\theta \log \pi_\theta(a_t|x_t) 
    \left( \nabla_\theta \log \pi_\theta(a_t|x_t) \right)^T b(x_t) \right],
\end{align*}
which in general does not evaluate to zero. In fact, the bias of $\nabla_\theta^2 J $ is clear if we take the special case $H=1$. In this case, it is more straightforward to derive
\begin{align*}
   \mathbb{E}_{\pi_\theta}\left[\nabla_\theta^2 J\right]  = \mathbb{E}_{\pi_\theta}\left[ 2 \nabla_\theta \log \pi_\theta(a_0|x_0) \left(\nabla_\theta \log \pi_\theta(a_0|x_0)\right)^T b(x_0) \right].
\end{align*}
As a result, we showed that TMAML \citep{liu2019taming} might achieve variance reduction by introducing baselines to the Hessian estimates of DiCE \citep{foerster2018dice}, but at the cost of bias.

\section{Further discussions on the assumptions}
\label{appendix:assumption}

In this section, we examine if the assumptions \textbf{(A.1)} and \textbf{(A.2)} are realistic. 

Both assumptions depend on particular functional form of the estimates $\hat{V}^{\pi_\theta}$. In general, we might assume that the estimates do not explicitly depend on $\theta$. Instead, they depend on $\theta$ via $\pi_\theta$. This involves a two-stage parameterization: $\theta\mapsto \pi_\theta$ and $\pi_\theta\mapsto \hat{V}^{\pi_\theta}$.
The two assumptions \textbf{(A.1)} and \textbf{(A.2)} can be realized by imposing constraints on these two parameterizations, as well as the off-policyness of $\pi_\theta$ relative to $\mu$, as discussed below.

\paragraph{Off-policyness.} In general, we might want to assume the ratios are bounded $\rho_t^\theta<R$ for constant some $R<\infty$. This is a common assumption. In our framework, we usually apply the estimates within a trust region optimization algorithm \citep{schulman2015,schulman2017}, this naturally proIS duces a bound on the ratios.

\paragraph{Parameterization $\theta\mapsto \pi_\theta$.} We seek parameterizations where $\nabla_\theta \rho_t^\theta$ are bounded. This can be achieved by bounding $\rho_t^\theta$ and $\nabla_\theta \log \pi_\theta(a|x)$. If we consider a tabular representation with softmax parameterization $\pi(a|x)\propto \exp(\theta(x,a))$. Under this parameterization, we can show $|\nabla_\theta^m \log \pi_\theta(a|x)|<M$ are bounded for all $(x,a)$ and all $\theta$.

\paragraph{Parameterization $\pi_\theta \mapsto \hat{V}^{\pi_\theta}$.} 
We want this parameterization to be sufficiently smooth.
In the examples we consider, TayPO-$K$ clearly satisfies this assumption because it is a polynomial in $\pi_\theta$. For DR, this assumption is satisfied when the MDP terminates within a finite horizon of $H<\infty$, such that the estimator contains polynomials of $\pi_\theta$ with order at most $H$.

\section{Proof of results in the main paper}
\label{appendix:proof}

\offpolicyestimate*

\begin{proof}
The two assumptions along with the unbiasedness of the estimates, allow us to exchange $m$\textsuperscript{th}-order derivatives and the expectation, and hence leading to the unbiasedness of the derivate estimates. The proof is similar to the exchange techniques used in \citep{huang2020importance} to show the unbiasedness of the first-order derivatives of DR estimates.

We proceed the argument with induction. Assume we have
\begin{align*}
    \mathbb{E}_\mu\left[\nabla_\theta^i \hat{V}^{\pi_\theta}(x_0,g)\right] = \nabla_\theta^i V^{\pi_\theta}(x_0,g),
\end{align*}
for some $i$. To define the $(i+1)$-th order derivative, we differentiate further through the $i$-th order derivative. Consider some particular component of the $(i+1)$-th order derivative, obtained by taking the derivative with respect to variable $\theta_L$. We now denote this component of the $(i+1)$-th order derivative as $D^{(i+1)}_L[\theta]$ evaluated at $\theta$. Let $D^{(i)}[\theta]$ denote the $i$-th order derivative (a tensor) evaluated at $\theta$. Also define $e_L\in\mathbb{R}^{D}$ as the one-hot vector such as its $L$-th component is one. By definition,
\begin{align*}
    D_L^{(i+1)} \coloneqq \lim_{h\rightarrow 0}\frac{D^{(i)}[\theta+e_L\cdot h] - D^{(i)}[\theta]}{h},
\end{align*}
we also denote the unbiased estimate to $D^{(i)}[\theta ]$ as $\hat{D}^{(i)}[\theta]$. The new estimate is
\begin{align*}
    \hat{D}_L^{(i+1)} \coloneqq \lim_{h\rightarrow 0}\frac{\hat{D}^{(i)}[\theta+e_L\cdot h] - \hat{D}^{(i)}[\theta]}{h}.
\end{align*}
Now we seek to establish that $\mathbb{E}_\mu\left[\hat{D}^{(i+1)}_L\right]=D^{(i+1)}_L$. Note that this is equivalent to showing
\begin{align*}
    \mathbb{E}_\mu\left[\lim_{h\rightarrow 0}\frac{\hat{D}^{(i)}[\theta+e_L\cdot h] - \hat{D}^{(i)}[\theta]}{h}\right]  = D_L^{(i+1)}.
\end{align*}
Note that with the RHS, we can use the definition along with unbiasedness of the $i$-th order derivatives
\begin{align*}
    D_L^{(i+1)} = \lim_{h\rightarrow 0 }\frac{\mathbb{E}_\mu\left[\hat{D}^{(i)}[\theta+e_L\cdot h]\right] - \mathbb{E}_\mu\left[\hat{D}^{(i)}[\theta]\right]}{h}
\end{align*}
Combining the new RHS into a single expectation, \textbf{(A.1)} entails the application of the mean value theorem,
\begin{align*}
    \lim_{h\rightarrow 0} \mathbb{E}_\mu\left[ \frac{\hat{D}^{(i)}[\theta+e_L\cdot h] - \hat{D}^{(i)}[\theta]}{h}\right]  = \lim_{h\rightarrow 0} \mathbb{E}_\mu \left[\hat{D}^{(i+1)}[\theta+e_L\cdot h\cdot \eta]\right],
\end{align*}
for some $\eta\in(0,1)$. Due to \textbf{(A.2)}, we can use dominated convergence theorem to exchange the limit and the expectation,
\begin{align*}
     \lim_{h\rightarrow 0} \mathbb{E}_\mu \left[\hat{D}^{(i+1)}[\theta+e_L\cdot h\cdot \eta]\right] =  \mathbb{E}_\mu  \left[\lim_{h\rightarrow 0}\hat{D}^{(i+1)}[\theta+e_L\cdot h\cdot \eta]\right] = \mathbb{E}_\mu\left[\hat{D}^{(i+1)}[\theta]\right].
\end{align*}
This proves the case for $(i+1)$-th order derivative. The base case holds for $i=0$ and we have the assumptions hold for all $0\leq i\leq m-1$. This concludes the proof of the theorem.

\end{proof}

\offpolicyestimategrad*
\begin{proof}
Starting from the definition of the DR estimate in Eqn~\eqref{eq:db-estimate}, which we recall here
\begin{align*}
    \hat{V}^{\pi_\theta}_\text{DR}(x_t,g) = Q(x_t,\pi_\theta(x_t),g) +   \rho_t^\theta \delta_t + \gamma \rho_t^\theta \left(\hat{V}^{\pi_\theta}_\text{DR}(x_{t+1},g) - Q(x_{t+1},\pi_\theta(x_{t+1}),g)\right).
\end{align*}

Note that both sides of the equations are functions of $\pi_\theta$. Since the DR estimate holds for all $\pi_\theta$ (assuming $\mu$ has larger support than $\pi_\theta$), and both sides are differentiable functions of $\theta$. We can differentiate both sides of the equation with respect to $\theta$, to yield its $m$\textsuperscript{th}-order derivatives. This produces the gradient estimates and the Hessian estimates accordingly, both in recursive forms.

Since \citep{huang2020importance} already provides a similar derivation in the first-order case,  we focus on the second-order. Given Eqn~\eqref{eq:recursive-pg}, we can further differentiate both sides of the equation by $\theta$. The RHS has three terms from Eqn~\eqref{eq:recursive-pg}. We rewrite the expression:

\paragraph{The first term.} This term produces a single term
$
    \nabla_\theta^2 Q(x_t,\pi_\theta(x_t),g)
$.

\paragraph{The second term.} Note a few useful facts: $\nabla_\theta \rho_t^\theta=\rho_t^\theta \nabla_\theta \log \pi_t$, $\nabla_\theta \delta_t = \gamma \nabla_\theta \hat{V}_\text{DR}^{\pi_\theta}(x_{t+1},g)$. This produces
\begin{align*}
    \nabla_\theta \left(\rho_t^\theta\delta_t \nabla_\theta\log\pi_t\right) = \rho_t^\theta \nabla_\theta \log \pi_t \delta_t \nabla_\theta \log \pi_t^T + \gamma \rho_t^\theta \nabla_\theta \log \pi_t \nabla_\theta \hat{V}_\text{DR}^{\pi_\theta}(x_{t+1},g)^T + \rho_t^\theta \delta_t \nabla_\theta^2 \log \pi_t.
\end{align*}
\paragraph{The third term.} Finally, the third term produces 
\begin{align*}
\gamma\rho_t^\theta\nabla_\theta^2 \hat{V}_\text{DR}^{\pi_\theta}(x_{t+1},g) + \gamma\rho_t^\theta \nabla_\theta   \hat{V}_\text{DR}^{\pi_\theta}(x_{t+1},g) \nabla_\theta \log \pi_t^T. \end{align*} 
Combining all three terms produces the recursive form of the DR Hessian estimates.

\end{proof}

\taypogradonpolicy*
\begin{proof}

We can express the residual of the derivatives as 
\begin{align*}
    \mathbb{E}_\mu\left[\nabla_\theta^m \hat{V}^{\pi_\theta}_{K}(x_0,g)\right]-\nabla_\theta^m V^{\pi_\theta}(x_0,g) = \nabla_\theta^m \left(\hat{V}^{\pi_\theta}_{K}(x_0,g) - V^{\pi_\theta}(x_0,g) \right) = \nabla_\theta^m \left( \sum_{k=K+1}^\infty U_k^{\pi_\theta}(x_0,g)\right).
\end{align*}
Above, the first equality comes from the unbiasedness of the estimates as well as the exchangability between derivatives and expectations, following similar arguments as those in Proposition~\ref{prop:offpolicyestimate}. The second equality comes from the Taylor expansion equality in Proposition~\ref{prop:taypo}. By the assumption that the MDP terminates within an horizon of $H<\infty$, we deduce that the summation contains at most $H-K$ terms and it is valid to exchange derivatives and the summation. Eqn~\eqref{eq:taypo-grad} is hence proved by applying the triangle inequality.
\begin{align*}
    \left\lVert \mathbb{E}_\mu\left[\nabla_\theta^m \hat{V}^{\pi_\theta}_{K}(x_0,g)\right]-\nabla_\theta^m V^{\pi_\theta}(x_0,g) \right\rVert_\infty \leq \sum_{k=K+1}^\infty \left\lVert \nabla_\theta^m U_k^{\pi_\theta}(x_0,g) \right\rVert_\infty.
\end{align*}
When on-policy, we plug in $\pi_\theta=\mu$. Since $K+1>m$, this implies that
after differentiating $U_k^{\pi_\theta}(x_0,g)$ for a total of $m$ times, each term contains at least $k-m$ terms of $\rho_t^\theta-1$ for some $t$. Since $K+1>m$, this means $U_k^{\pi_\theta}(x_0,g)=0$ for all indices $k$ within the summation. Hence we have zeros on the RHS and this shows Eqn~\eqref{eq:taypo-grad-onpolicy}.
\end{proof}

\section{Further details on sampled-based TayPO-$K$ estimates}
\label{appendix:taypo}

Please refer to the TayPO \citep{tang2020taylor} paper for further theoretical discussions. By definition, the $K$\textsuperscript{th}-order increment is
\begin{align*}
    U_K^{\pi_\theta}(x_0,g) \coloneqq \mathbb{E}_{\mu} \left[ \underbrace{ \sum_{t_1=0}^\infty \sum_{t_2=t_1+1}^\infty ...\sum_{t_K=t_{K-1}+1}^\infty \gamma^{t_K} \left(\Pi_{i=1}^K (\rho_{t_i}^\theta-1)\right) Q^\mu(x_{t_K},a_{t_K},g)}_{\hat{U}_K^{\pi_\theta}(x_0,g)} \right].
\end{align*}
Assume the trajectory is of finite length $T$ (or we can use the effective horizon $T_\gamma=1/(1-\gamma)$. The naive Monte-Carlo estimate $\hat{U}_K^{\pi_\theta}(x_0,g)$ consists of $O(T^K)$ terms. Since we usually care about $K\leq 2$, computing such a term exactly might still be tractable, as is shown later in Appendix~\ref{appendix:code}. We show in Algorithm 4 how to compute the estimates with complexity $O(T^2)$.

However, in some applications, we might seek to construct the estimates with better time complexity. The high-level idea is to achieve this through sub-sampling. Define $p_\gamma^{\pi}(x'|x)\coloneqq (1-\gamma)\sum_{t\geq 0}\gamma^t P_{\pi}(x_t=x'|x_0=x)$, where $P_{\pi}$ is the probability measure induced by $\pi$ and the MDP. We can rewrite the above into the following
\begin{align*}
    U_K^{\pi_\theta}(x_0,g) = \mathbb{E}_{t_1,t_2...t_K}\left[\left(\Pi_{i=1}^K (\rho_{t_i}^\theta-1)\right) Q^\mu(x_{t_K},a_{t_K},g)\right],
\end{align*}
where the sequence of states are sampled  
as $x_{t_{i+1}} \sim p(\cdot|x_i',a_i'), a_i'\sim \mu(\cdot|x_i'), x_i'\sim  p_\gamma^\mu(\cdot|x_{t_i})$ and $x_{t_0}=x_0$. Note that the above procedure could be achieved by first generating a full trajectory under $\mu$, and then sub-sampling random times along the trajectory. As such, the estimate takes linear time to compute, at the cost of potentially larger variance.

\section{High-level code for implementations of Hessian estimates and meta-gradients}
\label{appendix:code}

Below, we introduce a few important details on how to convert off-policy evaluation estimates into Hessian estimates, with the help of auto-diff.

\subsection{Implementing Hessian estimates with off-policy evaluation subroutines}
In Figure~\ref{fig:code}, we show a high-level JAX implementation of estimating Hessians using off-policy evaluiation subroutines. The pseudocode assumes access to some well-established off-policy evaluation functions, as are commonly implemented in off-policy RL algorithms such as ACER \citep{wang2016}, Retrace \citep{munos2016safe}, IMPALA \citep{espeholt2018impala}, R2D2 \citep{kapturowski2018recurrent} and so on. The function needs to be written in auto-differentiation libraries, such that after the computations, value function estimates' parameter dependencies are naturally maintained. Then this pipeline could be directly implemented as part of a meta-RL algorithm. 

\begin{figure}[h]
    \centering
    \includegraphics[keepaspectratio,width=.95\textwidth]{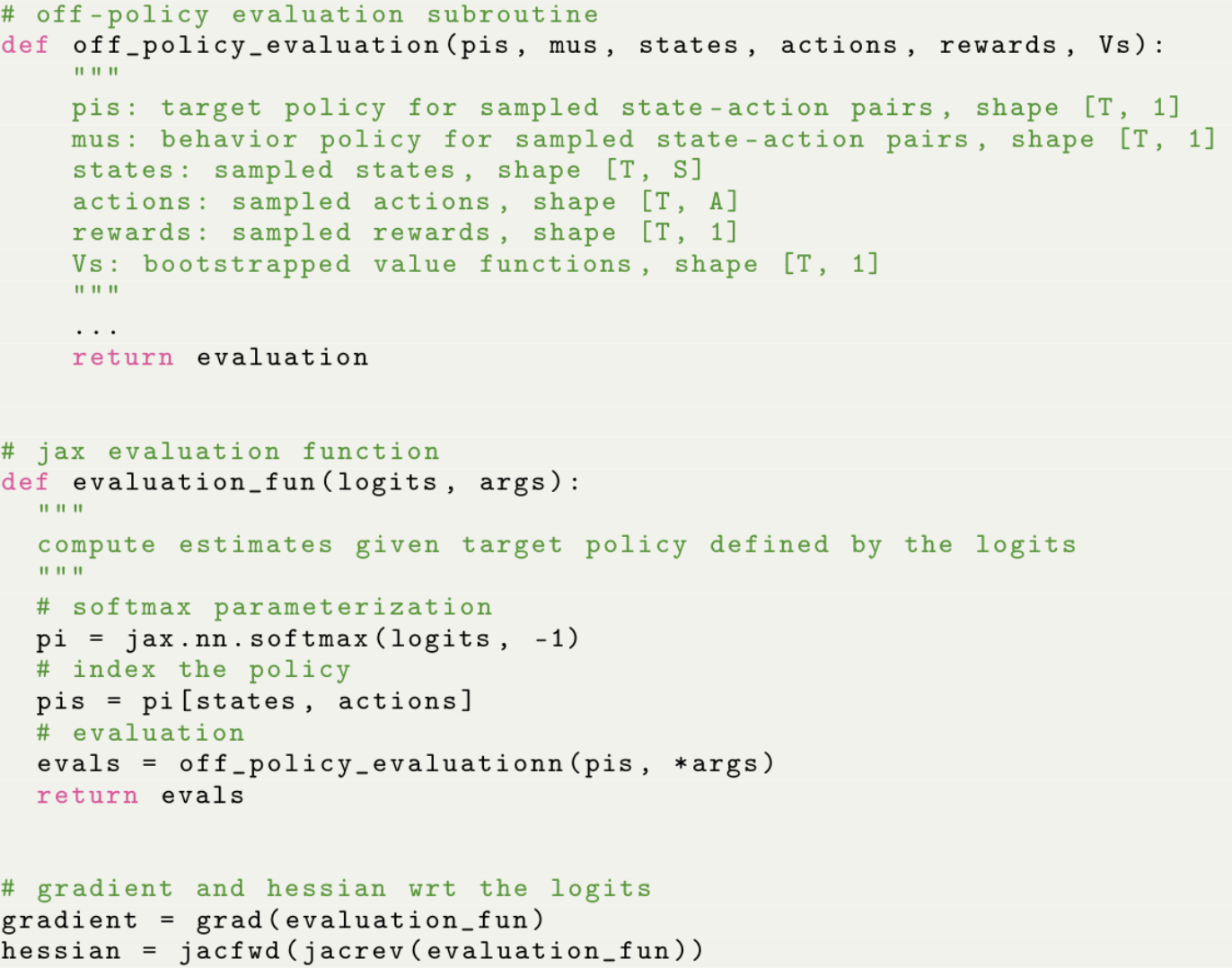}
    \caption{JAX-based high-level code for the implementation of Hessian estimates. We can easily convert any established trajectory-based off-policy evaluation subroutine into estimates of Hessian matrix, by auto-differentiating through the estimates. This can be implemented in any auto-differentiation frameworks. }
    \label{fig:code}
\end{figure}

\subsection{Examples of off-policy evaluation estimates}

The following off-policy evaluation estimates can be abstracted as functions that take as input: a partial trajectory $(x_t,a_t,r_t)_{t=0}^{T}$ of length $T+1$, the target policy $\pi_\theta$ and a behavior policy $\mu$. Optionally, the function could also take as input some critic function $Q$. We detail how to compute certain estimates below.

\paragraph{DR estimates.} In Algorithm 2, we provided the pseudocode for computing DR estimates. The step-wise IS estimates could be computed as a special case by setting $Q=0$.

\paragraph{TayPO-$1$ estimates.} See Algorithm 3 for details.

\begin{algorithm}[h]
\label{algo:evaluation-subroutine2}
\begin{algorithmic}
\REQUIRE \textbf{Inputs}: Trajectory $(x_t,a_t,r_t)_{t=0}^T$, target policy $\pi_\theta$, behavior policy $\mu$, (optional) critic $Q$.
\STATE Initialize $\hat{V}=Q(x_T,\pi_\theta(x_T),g)$.
\STATE Compute IS ratio $\rho_t^\theta=\pi_\theta(a_t|x_t)/\mu(a_t|x_t)$.
\STATE Compute Q-function estimates for all $Q^\mu(x_t,a_t)$. This could be done by computing $\hat{Q}^\mu(x_t,a_t)=\sum_{s\geq t} r_s\gamma^{s-t} + \gamma^{T-t} Q(x_T,a_T)$.
\STATE Compute the estimate
$\hat{V}=\hat{Q}^\mu(x_0,a_0) + \sum_{t=0}^T \gamma^t (\rho_t^\theta-1) \hat{Q}^\mu(x_t,a_t)$.
\STATE Output $\hat{V}$ as an estimate to $V^{\pi_\theta}(x_0,g)$.
\caption{Example: an off-policy evaluation subroutine for computing the TayPO-$1$ estimate}
\end{algorithmic}
\end{algorithm}

\paragraph{TayPO-$2$ estimates.}
See Algorithm 4 for details.

\begin{algorithm}[h]
\label{algo:evaluation-subroutine2}
\begin{algorithmic}
\REQUIRE \textbf{Inputs}: Trajectory $(x_t,a_t,r_t)_{t=0}^T$, target policy $\pi_\theta$, behavior policy $\mu$, (optional) critic $Q$.
\STATE Initialize $\hat{V}=Q(x_T,\pi_\theta(x_T),g)$.
\STATE Compute IS ratio $\rho_t^\theta=\pi_\theta(a_t|x_t)/\mu(a_t|x_t)$.
\STATE Compute Q-function estimates for all $Q^\mu(x_t,a_t)$. This could be done by computing $\hat{Q}^\mu(x_t,a_t)=\sum_{s\geq t} r_s\gamma^{s-t} + \gamma^{T-t} Q(x_T,a_T)$.
\STATE Compute the first-order estimate
$\hat{V}_1=\hat{Q}^\mu(x_0,a_0) + \sum_{t=0}^T \gamma^t (\rho_t^\theta-1) \hat{Q}^\mu(x_t,a_t)$.
\STATE Compute the second-order estimate
$\hat{V}_2=\sum_{t=0}^T \sum_{s=t+1}^T \gamma^s (\rho_t^\theta-1)(\rho_s^\theta-1) \hat{Q}^\mu(x_s,a_s)$.

\STATE Output $\hat{V}_1+\hat{V}_2$ as an estimate to $V^{\pi_\theta}(x_0,g)$.
\caption{Example: an off-policy evaluation subroutine for computing the TayPO-$2$ estimate}
\end{algorithmic}
\end{algorithm}

\paragraph{Truncated DR estimates.} See Algorithm 5 for more details. Truncated DR is similar to DR except that the IS ratio is replaced by truncated IS ratios $\text{min}(\rho_t^\theta,\bar{\rho})$ for some $\rho$. For the experiments we set $\rho=1$, inspired by V-trace  \citep{munos2016safe,espeholt2018impala}. The main motivation for the truncation is to control for the variance induced by IS ratios. However, this also introduces bias into the estimates, unless the samples are near on-policy.

\begin{algorithm}[h]
\label{algo:evaluation-subroutine}
\begin{algorithmic}
\REQUIRE \textbf{Inputs}: Trajectory $(x_t,a_t,r_t)_{t=0}^T$, target policy $\pi_\theta$, behavior policy $\mu$, (optional) critic $Q$.
\STATE Initialize $\hat{V}=Q(x_T,\pi_\theta(x_T),g)$.
\FOR{$t=T-1,\dots0$}
\STATE Compute truncated IS ratio $\tilde{\rho}_t^\theta=\min(\pi_\theta(a_t|x_t)/\mu(a_t|x_t),\rho)$ for some $\rho>0$.
\STATE Recursion: $\hat{V}\leftarrow Q(x_t,\pi_\theta(a_t),g) + \gamma \tilde{\rho}_t^\theta (r_t + \gamma Q(x_{t+1},\pi_\theta(x_{t+1}),g) - Q(x_t,a_t)) + \gamma \tilde{\rho}_t^\theta \hat{V}$.
\ENDFOR
\STATE Output $\hat{V}$ as an estimate to $V^{\pi_\theta}(x_0,g)$.
\caption{Example: an off-policy evaluation subroutine for computing the truncated DR estimate}
\end{algorithmic}
\end{algorithm}

\subsection{Implementing meta-RL estimates}
To implement meta-RL estimates in a auto-differentiation framework, the aim is to construct a single scalar objective, the auto-differentiation of which produces an estimate to the meta-gradient.

Let $\hat{V}_\text{inner}(\theta,D)$ be an inner loop objective one can use to construct the inner loop adaptation steps. In this case, $\hat{V}_\text{inner}(\theta,D)$ can take as input: the parameter $\theta$ and some data $D$. Here, for example, the data $D$ might consist of sampled trajectories and other hyper-parameters such as discount factors. In our paper, this objective can be replaced by any off-policy evaluation estimates. The updated parameter is computed as: $\theta'=\theta+\eta\nabla_\theta \hat{V}_\text{inner}(\theta,D)$.

The meta objective is defined as the value function, or equivalently some outer loop objective $\hat{V}_\text{outer}(\theta,D)$ that also takes as input the parameter $\theta$ and some data $D$. The overall objective can be defined as:
\begin{align*}
    \hat{V}_\text{outer}\left(\theta+\eta \nabla_\theta \hat{V}_\text{inner}(\theta,D)\right).
\end{align*}
Auto-differentiaing through the above objective can produce estimates to meta-gradients. This objective is also easy to implement in auto-differentiation frameworks.

\section{Experiment}
\label{appendix:exp}
Below, we introduce further details in the experiments.

\subsection{Tabular MDP}

\paragraph{MDPs.} These MDPs have $|\mathcal{X}|=10$ states and $|\mathcal{A}|=5$ actions. The transition probabilities $p(\cdot|x,a)$ are generated from independent Dirichlet distributions with parameter $(\alpha,...\alpha)\in\mathbb{R}^{|\mathcal{X}|}$. Here, we set $\alpha=0.001$. The discount factor is $\gamma=0.8$ for all problems. 

\paragraph{Setups.} The policy $\pi_\theta$ is parameterized as $\pi_\theta(a|x)=\exp(\theta(x,a))/\sum_{b} \exp(\theta(x,b))$. The behavior policy $\mu$ is uniform and $\theta$ is set such that $\theta(x,a)=\log \pi(a|x)$ where $\pi = (1-\epsilon) \mu+\epsilon \pi_d$ for some deterministic policy $\pi_d$ and parameter $\epsilon\in [0,1]$. 

\paragraph{Experiments.} In each experiment, we generate a random MDP and initialize the policy. The agent colelcts $N$ trajectories of length $T=20$ (such that $\gamma^T\approx 0$) from a fixed initial state $x_0$. We then compute gradient and Hessian estimates of the initial state $V^{\pi_\theta}(x_0)$ by directly diffrentiating through various $N$-trajectory off-policy evaluation estimates: $\nabla_\theta^m \hat{V}^{\pi_\theta}(x_0)$ for $m=1,2$. We also calculate the ground truth gradient and Hessian using transition probabilities of the MDPs.

\paragraph{Accuracy measure.} Given an estimate $x\in\mathbb{R}^L$ and a ground truth value $y\in\mathbb{R}^L$, we measure the accuracy between the two tensors as:
\begin{align*}
    \text{Acc}(x,y)\coloneqq \frac{x^T y }{ \sqrt{x^Tx}\sqrt{y^Ty}}.
\end{align*}
Note that this measure is bounded between $[-1,1]$. This advantage of this measure is that it neglects the absolute scales of the tensors, i.e., if $x=ky$ for some scalar $k\neq 0$, then $\text{Acc}(x,y)=1$. This metric is used in a number of prior work \citep{foerster2018dice,farquhar2019loaded,mao2019baseline} and is potentially a more suitable measure of accuracy given that in large scale experiments, downstream applications typically use adaptive optimizers.

\paragraph{Further results: effect of sample size.}
In Figure~\ref{fig:tabular}(a), we fix the level of off-policyness $\epsilon=0.5$ and show the estimates as a function sample size $N$. As $N$ increases, the accuracy measures of most estimates increase. Intuitively, this is because the variance decreases while the bias is not impacted by the sample size. Comparing the step-wise IS estimate with the DR estimate, we see that the DR estimate generally performs better due to variance reduction. This is consistent with findings in prior work \citep{farquhar2019loaded,mao2019baseline}). Further, it is worth noting the  first-order estimate performs quite well when the $N$ is small, outperforming the DR estimate. This implies the importance of controlling the number of step-wise IS ratios for further variance reduction. However, as $N$ increases, the performance of the first-order estimate does not improve as much compared to other alternatives, and is finally surpassed by the DR estimate mainly due to bigger bias. 

Consider the second-order estimate. When $N$ is small, the second-order estimate slightly underperforms the first-order estimate. This is expected because at small sample sizes, the variance dominates. However, as $N$ increases, its performance quickly tops across all different estimates, including the DR estimate. Overall, we expect the second-order estimate to achieve a better bias-variance trade-off, especially in the medium data regime. This should be more significant in large-scale setups where estimation horizons are longer and variance dominates further. 

\begin{figure}[t]
    \centering
    \subfigure[Gradient - sample size ]{\includegraphics[keepaspectratio,width=.48\textwidth]{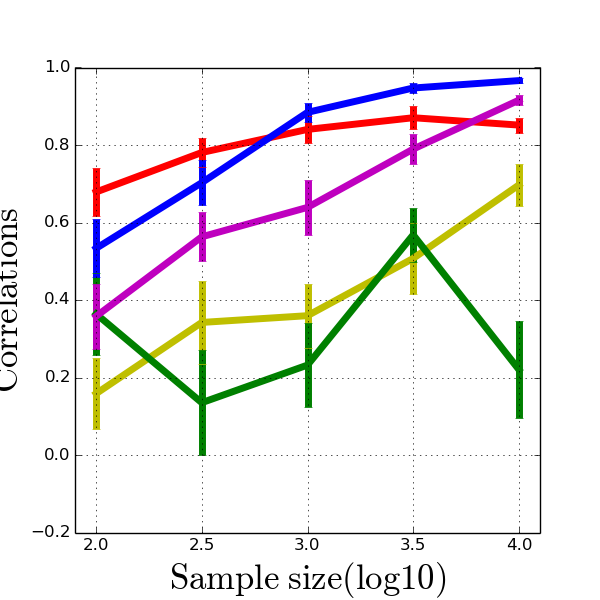}}
    \caption{Performance measure as a function of sample size . Each plot shows the accuracy measure between the estimates and the ground truth.  Overall, the second-order estimate achieves a better bias and variance trade-off. Here, the plot shows results for estimating gradients.}
    \label{fig:tabular-appendix}
\end{figure}

\subsection{Large scale experiments}

In large scale experiments, including the continuous 2-D navigation environments and simulated locomotion environments, we adopt the following setups.

\paragraph{Code base.} We adopt the code base released by \citep{rothfuss2018promp}. We make minimal changes to the code base, such that the second-order estimate is comparable to other algorithms under the established experimental setups. For missing hyper-parameters, please refer to the code base for further details. Importantly, note that in the original code base as well as the paper \citep{rothfuss2018promp}, the authors suggest the default learning rate of $\alpha=10^{-3}$, which we find tends to destabilize learning. Instead, we use $\alpha=10^{-4}$, which works more stably.

\paragraph{Computational resources.} All high-dimensional experiments are run on a computer cluster with multiple CPUs. Each separate experiment is run with $12$ CPUs as actors for data collection and parallel computations of parameter updates. The run time for each experiment is on average $36$ hours per experiment. For small experiments, we run them on a single laptop machine with $8$ CPUs.

\paragraph{Agent details.} The agent adopts a MLP architecture with $[64,64]$ hidden units. The agent takes in a state $x$ and outputs a Gaussian policy $a\sim \mathcal{N}(\mu_\theta(x),\sigma^2(x)) $ where $\mu_\theta,\sigma_\theta$ are parameterized  by the neural networks. The agent collects samples $n=40$ goals at each iteration to construct meta-gradient estimates; the inner loop adaptation is computed with step size $\eta=0.1$. Inner loop adaptations are computed with $B=20$ trajectories each of length $T=100$. All outer loop optimizers use the learning rate $\alpha=10^{-4}$. 
\paragraph{Algorithmic details.} The PROMP and PROMO-TayPO-$2$ enforces a soft trust region constraint through clipping
\begin{align*}
    \bar{\rho}_t^\theta = \text{clip}(\rho_t^\theta,1-\epsilon,1+\epsilon),
\end{align*}
where by default $\epsilon=0.3$. The PPO optimizers take $5$ gradient steps during each iteration.
All outer loop gradient based optimizers use Adam optimizers \citep{kingma2013}.

\paragraph{Summary of baselines.} The baselines include the following: TRPO-MAML  uses TRPO as the outer loop optimizer \citep{schulman2015} and the biased MAML implementation \citep{finn2017model}; TRPO-FMAML, which is short for first-order MAML,  approximates the Hessian matrix by an identity matrix \citep{finn2017model}; TRPO-EMAML augments the MAML loss function by an exploration bonus term, which effectively corrects for the bias introduced by vanilla MAML \citep{al2017continuous,stadie2018some}; TRPO-DiCE, which uses the DR estimate to implement the inner loop update, such that the Hessian estimates are unbiased \citep{foerster2018dice,farquhar2019loaded,mao2019baseline}.

Closely related to our new algorithm is the proximal meta policy search (PROMP) \citep{rothfuss2018promp}, which uses PPO as the outer loop optimizer \citep{schulman2017} and the first-order estimate (LVC estimate) as the inner loop loss function \citep{kakade2002approximately,tang2020taylor}. Our new algorithm is called PROMP-TayPO-$2$, which is a combination of PROMP and TayPO-$2$. The only difference between PROMP and PROMP-TayPO-$2$ is that the inner loop loss function is now implemented with the second-order estimate to alleviate the bias introduced by the first-order term.

\paragraph{Practical implementations of  second-order estimates.} We denote $\hat{V}_1$ as value function estimate based on the first-order approximation; and let $\hat{V}_2$ be the value function estimate. In practice, the second-order estimate we implement is a mixture between the two estimates with $\beta\in[0,1]$
\begin{align*}
    \hat{V}_\beta = (1-\beta)\hat{V}_1+\beta\hat{V}_2.
\end{align*}
This overall estimate is a convex combination of the two estimates. By moving $\beta=0$ to $\beta=1$, we interpolate between the first-order and the second-order estimate.
Throughout the large-scale experiments we find $\beta=0.3$ to work the best. We also show in the next section the sensitivity of performance to $\beta$ on 2-$D$ control environments.

\subsection{Ablation study}
\paragraph{Sensitivity to the second-order coefficient on 2-$D$ control environments.}
In Figure~\ref{fig:ball-appendix}(a), we show the sensitivity of the performance to the mixing coefficient $\beta\in[0,1]$. Note that when $\beta=0$, the estimate is exactly the first-order estimate; when $\beta=1$, the estimate recovers the full second-order estimate. We see that on the 2-$D$ control environment, as we increase $\beta$ from $0$ to $1$, the performance stably improves. When $\beta\approx 0.9$, it seems that the performance reaches a plateau, indicating that $\beta=0.9$ potentially achieves the best bias and variance trade-off between the two extremes. 

\paragraph{Sensitivity to off-policyness on 2-$D$ control environments.} In Figure~\ref{fig:ball-appendix}(b), we study the sensitivity of algorithms to off-policyness in high-d setups. In PROMP, the policies are optimized with behavior policy $\mu$, whose distance against the target policy is constrained by a trust region. The trust region is enforced softly via a clipping coefficient $\epsilon$ (see Appendix~\ref{appendix:exp} or \citep{rothfuss2018promp} for details). When $\epsilon$ increases, the effective level of off-policyness increases. We see that as $\epsilon$ increases, both first and second-order estimates degrade in performance. However, the second-order estimates perform more robustly, as similarly suggested in the toy example.

\begin{figure}[t]
    \centering
    \subfigure[2-$D$ control: second-order coefficient ]{\includegraphics[keepaspectratio,width=.48\textwidth]{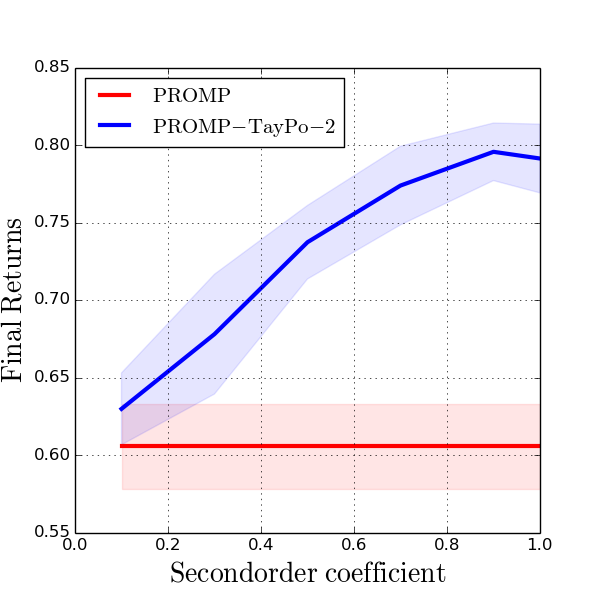}}
    \subfigure[2-$D$ control: off-policy ]{\includegraphics[keepaspectratio,width=.48\textwidth]{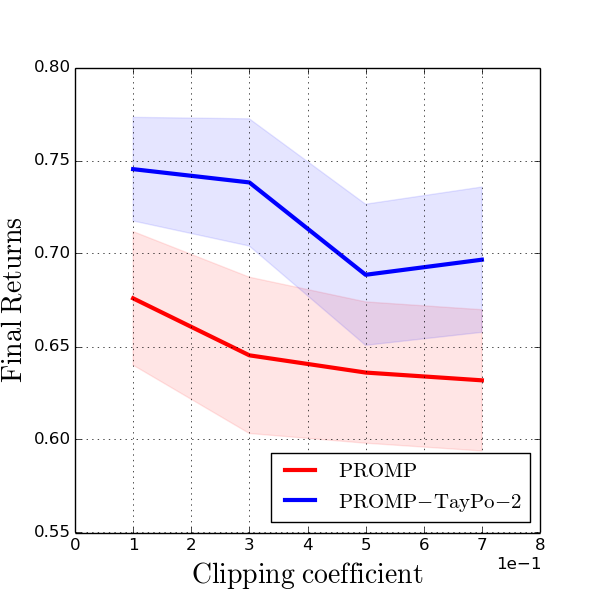}}
    \caption{Ablation study: (a) second-order coefficients; (b) off-policyness.  The above two plots show the sensitivity of the first and second-order estimate to hyper-parameters, for 2-$D$ control. The second-order estimate is generally more robust. All curves are averaged over $10$ runs. }
    \label{fig:ball-appendix}
\end{figure}

\end{document}